\documentclass{article}

\usepackage{microtype}
\usepackage{graphicx}
\usepackage{subfigure}
\usepackage{booktabs} 

\usepackage{hyperref}

\PassOptionsToPackage{table,xcdraw,dvipsnames}{xcolor}
\usepackage[accepted]{icml2025}

\usepackage{amsmath}
\usepackage{amssymb}
\usepackage{mathtools}
\usepackage{amsthm}

\usepackage[capitalize,noabbrev]{cleveref}

\theoremstyle{plain}
\newtheorem{theorem}{Theorem}[section]

\newtheorem{lemma}[theorem]{Lemma}

\newtheorem{fact}[theorem]{Fact}

\theoremstyle{definition}

\theoremstyle{remark}

\usepackage[textsize=tiny]{todonotes}

\usepackage{subcaption} 
\usepackage{bbm}
\usepackage{graphicx}
\usepackage{mathtools}
\usepackage{thm-restate}

\definecolor{gold}{rgb}{1,0.84,0}
\definecolor{silver}{rgb}{0.75,0.75,0.75}
\definecolor{bronze}{rgb}{0.8,0.5,0.2}
\definecolor{whitesilver}{rgb}{1,1,1}
\definecolor{whitebronze}{rgb}{1,1,1}
\definecolor{whitegold}{rgb}{1,1,1}

\DeclareMathOperator*{\argmin}{arg\,min}
\DeclareMathOperator*{\E}{\mathbb{E}}

\newcommand{\hide}[1]{}

\newcommand{\kshap}{Kernel SHAP}

\newcommand{\psemi}{\phi}

\newcommand{\phivec}{\boldsymbol{\phi}}
\newcommand{\approxphivec}{\tilde{\boldsymbol{\phi}}}

\icmltitlerunning{Kernel Banzhaf: A Fast and Robust Estimator for Banzhaf Values}

\begin{document}

\twocolumn[
\icmltitle{Kernel Banzhaf: A Fast and Robust Estimator for Banzhaf Values}



\icmlsetsymbol{equal}{*}
\icmlsetsymbol{intern}{\textsuperscript{\textdagger}}

\begin{icmlauthorlist}
\icmlauthor{Yurong Liu}{equal,intern,yyy}
\icmlauthor{R. Teal Witter}{equal,yyy}
\icmlauthor{Flip Korn}{comp}
\icmlauthor{Tarfah Alrashed}{comp}
\icmlauthor{Dimitris Paparas}{comp}
\\
\icmlauthor{Christopher Musco}{yyy}
\icmlauthor{Juliana Freire}{yyy}
\end{icmlauthorlist}
\icmlaffiliation{yyy}{New York University}
\icmlaffiliation{comp}{Google Research}

\icmlcorrespondingauthor{R. Teal Witter}{rtealwitter@nyu.edu}

\icmlkeywords{Machine Learning, ICML}

\vskip 0.3in
]



\printAffiliationsAndNotice{\icmlEqualContribution\; \icmlIntern} 

\begin{abstract}
Banzhaf values provide a popular, interpretable alternative to the widely-used Shapley values for quantifying the importance of features in machine learning models. Like Shapley values, computing Banzhaf values exactly requires time exponential in the number of features, necessitating the use of efficient estimators. Existing estimators, however, are limited to Monte Carlo sampling methods. In this work, we introduce Kernel Banzhaf, the first regression-based estimator for Banzhaf values. Our approach leverages a novel regression formulation, whose exact solution corresponds to the exact Banzhaf values. Inspired by the success of Kernel SHAP for Shapley values, Kernel Banzhaf efficiently solves a sampled instance of this regression problem. Through empirical evaluations across eight datasets, we find that Kernel Banzhaf significantly outperforms existing Monte Carlo methods in terms of accuracy, sample efficiency, robustness to noise, and feature ranking recovery. Finally, we complement our experimental evaluation with strong theoretical guarantees on Kernel Banzhaf's performance.
\end{abstract}

\section{Introduction}
The increasing complexity of AI models has intensified challenges associated with model interpretability. Modern machine learning models, such as deep neural networks and ensemble methods, often operate as
``opaque boxes." 
This opacity makes it difficult for users to understand and trust model predictions, especially in decision-making scenarios like healthcare, finance, and legal applications, which require rigorous justifications. Thus, there is a pressing need for reliable explainability tools to bridge the gap between complex model behaviors and human understanding.

Among the various methods employed within explainable AI, game-theoretic approaches have gained prominence for quantifying the contribution of features both to the overall performance of a machine learning model, and to individual predictions made by the model. 
The most well-known game-theoretic approach is based on \emph{Shapley values}, which provide a principled way to attribute the contribution of $n$ individual players to the outcome of a cooperative game, which is defined by a set function $v$ that maps every subset of $\{1, \ldots, n\}$ to a real value \citep{shapley:book1953}. 

In the context of feature attribution, each ``player'' is a feature and $v$ maps a subset of features to, e.g., the overall test loss when the chosen features are used, or the prediction made for a specific individual given access to the chosen features. 
The Shapley value quantifies the average marginal contribution of a feature on the set function, computed as the weighted average over all possible combinations of features included in the model \citep{lundberg2017unified}.
More formally, the Shapley value $\phi_i$ of a player $i \in \{1,\ldots,n\}$ is
\begin{align}\label{eq:shapley_def}
    \phi_i^{\textnormal{shap}} =
    \frac{1}{n} \sum_{S \subseteq [n] \setminus \{i\}} \binom{n-1}{|S|}^{-1} [v(S \cup \{i\}) - v(S)].
\end{align}

While primarily associated with feature attribution \citep{lundberg2017unified,karczmarz2022improved}, game-theoretic methods like Shapley values have also contributed to broader machine learning tasks beyond explainability, such as feature selection \citep{covert2020sage} and data valuation \citep{ghorbani2019datashapley,wang2023databanzhaf}. 

An alternative to Shapley values are Banzhaf values, which also compute each individual's contribution to an overall outcome~\citep{banzhaf1965,lehrer1988axiomatization}. 
While Shapley values are more widely used, 
Banzhaf values are often considered more intuitive for AI applications since they 
treat each subset of players as equally important, directly measuring the impact of each player across all possible combinations.
In particular, the Banzhaf value is
\begin{align}\label{eq:banzhaf_def}
    \phi_i^{\textnormal{banz}} = \frac1{2^{n-1}} \sum_{S \subseteq [n] \setminus \{i\}} [v(S \cup \{i\}) - v(S)].
\end{align}
In contrast, the Shapley values use non-uniform weight $\binom{n-1}{|S|}^{-1}$, which depends on the subset size.
It has also been observed that Banzhaf values tend to be more robust than Shapley values in the context of explainable AI \citep{karczmarz2022improved,wang2023databanzhaf}. These benefits have led to the broad use of Banzhaf values for feature attribution \citep{datta2015,kulynych2017,sliwinski2018,patel2021high} and, more recently, for data valuation \cite{wang2023databanzhaf,li2024robust}.

For general set functions, the exact computation of both Shapley and Banzhaf values requires at least $2^n$ time, where $n$ is the number of features in our problem. In particular, we must evaluate $v(S)$ for every possible subset $S$ of $\{1,\ldots, n\}$. Accordingly, the problem of \emph{approximating} Shapley values at lower cost has been widely studied, especially for model explanation. A leading method for Shapley values is \kshap\ \citep{lundberg2017unified}, a model-agnostic technique that leverages a connection to linear regression, approximating Shapley values by solving a sampled weighted least squares problem \citep{charnes1988extremal, lundberg2017unified}.
\kshap\ has been further improved with paired sampling \citep{covert2020improving} and leverage score sampling \cite{musco2024leverage}.

In contrast to Shapley values, only a few algorithms have been proposed to approximate Banzhaf values for arbitrary set functions.
For tree-structured set functions, such as those corresponding to decision tree based models, exact Banzhaf values can be efficiently computed \citep{karczmarz2022improved}.
For general set functions, like those corresponding to neural networks models, Monte Carlo sampling (i.e., estimating the sum in \eqref{eq:banzhaf_def} via a subsample) can be used to estimate each Banzhaf value separately \citep{merrill1982approximations,bachrach2010approximating}.
The Maximum Sample Reuse (MSR) algorithm is a variant of Monte Carlo sampling that reuses samples for the estimates of different Banzhaf values \citep{wang2023databanzhaf}. However, even with reuse, the naive Monte Carlo estimator typically requires a large number of samples (which equates to a large number of evaluations of $v$) to obtaining meaningfully accurate estimates for the Banzhaf values. 

\paragraph{Our Contributions.} In this work, we show for the first time that, like Shapley values, Banzhaf values can be approximated far more accurately through \emph{approximate regression}. To do so, we introduce a novel regression formulation (concretely, a linear least squares regression problem involving $n$ variables) whose solution exactly corresponds to the Banzhaf values (Section~\ref{sec:lr}). We then show that this problem can be efficiently solved by subsampling rows, each of which requires just a single evaluation of the set function $v$. 

The ultimate payoff is a novel estimator, \emph{Kernel Banzhaf}, that on all datasets tested, far outperforms existing Monte Carlos estimation methods (see \Cref{sec:experiments} for more details). Concretely, for a fixed number of evaluations of $v$, our method typically improves on the accuracy of Banzhaf values estimation by an order of magnitude or more.

While two alternative regression formulations were previously known for Banzhaf values, neither lead to efficient algorithms. One only recovers the values up to an unknown additive shift \cite{ruiz1998family,li2024one} and unfortunately, it is unclear how to estimate this shift efficiently.
The other exactly recovers the Banzhaf values, but only for binary and monotone set functions \cite{hammer1992approximations}. Our formulation, in contrast, exactly recovers the Banzhaf values without any assumptions on the set function.

We conduct an extensive experimental evaluation of Kernel Banzhaf in comparison to existing Banzhaf estimators, assessing performance in terms of sample efficiency and robustness to noise. We compare accuracy using both the relative squared $\ell_2$-norm error and an objective naturally suggested by our novel regression problem. Furthermore, we evaluate Kernel Banzhaf in a feature ranking task, where the goal is to accurately recover the most important features. Our experiments span diverse model classes, including tree-based models and neural networks. Our findings suggest that Kernel Banzhaf is substantially more accurate and robust than prior estimators (see e.g., Figure~\ref{fig:exact-vs-estimate}).

We complement our experimental evaluation of Kernel Banzhaf with strong theoretical guarantees.
We prove that, with just $O(n \log n)$ evaluations of $v$, Kernel Banzhaf returns a solution with bounded squared $\ell_2$-norm difference to the exact Banzhaf values (see \Cref{coro:l2norm}).
Obtaining similar theoretical guarantees for Monte Carlo sampling and improvements like Maximum Sample Reuse requires making a strong assumption on the maximum magnitude of $v(S)$. See Section \ref{sec:guarantees} for further comparison of theoretical results. 

\begin{figure*}
    \centering
    \includegraphics[width=\linewidth]{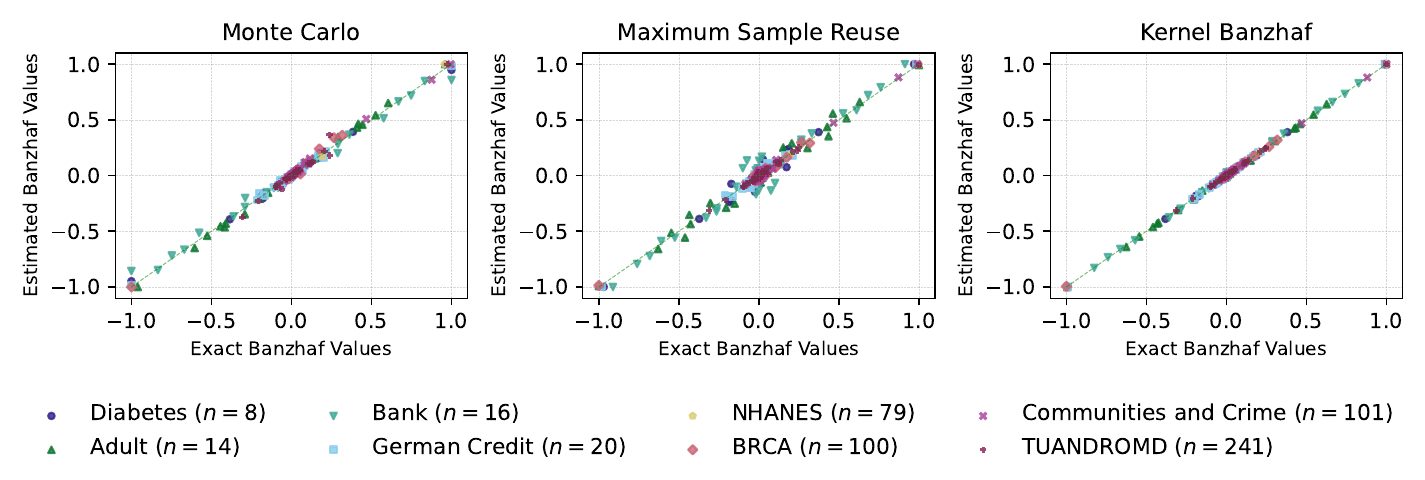}
    \vspace{-.5cm}
    \caption{Comparison of exact and estimated Banzhaf feature attribution values across eight datasets, where every estimator uses $m=20n$ evaluations of the set function $v$. The evaluations dominate the computational cost of each algorithm. Each subplot, labeled with its estimator, shows normalized estimated versus exact Banzhaf values across all features for a randomly selected data point from each dataset. Points closer to the diagonal line indicate more accurate estimates; the plots suggest that Kernel Banzhaf is more accurate than the Monte Carlo and Maximum Sample Reuse estimators.}
    \label{fig:exact-vs-estimate}
    \vspace{-.3cm}
\end{figure*}

Our main contributions can be summarized as follows: 
\begin{enumerate}
    \item We propose Kernel Banzhaf, the first regression-based approximation algorithm for estimating Banzhaf values for general set functions. The foundation of the algorithm is a novel regression formulation, whose exact solution is the Banzhaf values.
    \item We show that Kernel Banzhaf provably returns accurate solutions. Our theoretical guarantees depend on the magnitude of the true Banzhaf values rather than the magnitude of the set function, as for prior estimators.
    \item We extensively evaluate Banzhaf estimator performance across eight popular datasets, in various settings, and under several measures of performance. Our findings demonstrate that Kernel Banzhaf significantly outperforms prior work.
    
    
\end{enumerate}

\section{Background}
\label{sec:background}

Let $n$ be a positive integer and define $[n] \vcentcolon = \{1,2, \ldots, n\}$.
We let $2^{[n]}$ denote the set of all $2^n$ possible subsets of $[n]$ (including the emptyset).
For set $S$, define the indicator vector $\mathbbm{1}[S] \in \{0,1\}^n$ so that $\mathbbm{1}[S]_i = 1$ if $i \in S$ and $0$ otherwise.
Let $v: 2^{[n]} \to \mathbb{R}$ be a set function.
Shapley and Banzhaf values are as defined in Equations \ref{eq:shapley_def} and \ref{eq:banzhaf_def} in the prior section.
For notational simplicity, we will use $\phi_i = \phi_i^{\textnormal{banz}}$ in the rest of the paper.
Let $\phivec \in \mathbb{R}^{n}$ denote the vector $[\psemi_1, \ldots, \psemi_n]$.

Since we make no structural assumptions about $v$, we consider estimators for computing Banzhaf values that access the function via \textit{samples}; formally, a single sample corresponds to evaluating $v(S)$ for any $S$ the algorithm chooses.

\textbf{Banzhaf Values for Feature Attribution.} As discussed in the previous section, Shapley and Banzhaf values have a number of applications in machine learning, which correspond to different choices of the set function $v$.
For example, we can consider feature selection by defining $v(S)$ as the loss of a model trained only on the features in $S$.
Similarly, we can consider data selection by defining $v(S)$ as the loss of a model trained only on the observations in $S$.


One of the most popular applications of Banzhaf values in explainable AI is \emph{feature attribution} (also known as feature importance and feature influence).
Let $\mathcal{M}: \mathbb{R}^n \to \mathbb{R}$ be a trained model which takes input $\mathbf{x} \in \mathbb{R}^n$.
Our goal is to explain the predictions of the model on a particular data point $\mathbf{x}$.
Given a subset of features $S$, define $\mathbf{x}^S$ as the observation where $\mathbf{x}^S_i = \mathbf{x}_i$ if feature $i \in S$ and, otherwise, $\mathbf{x}^{S}_i$ is sampled from the data distribution (possibly conditioned on the features in $S$).
Then $v(S) = \mathbb{E}[\mathcal{M}(\mathbf{x}^S)]$, where the expectation is over the sampling of the features $i \notin S$. This expectation can be estimated via sampling, although for tree-based models, it is actually possible to \emph{exactly compute} the Banzhaf feature attribution values in an efficient way \cite{karczmarz2022improved}. This is valuable in evaluating the performance of Banzhaf value approximation methods since, for these models, we can obtain the true Banzhaf values to compare against.

\section{Kernel Banzhaf}

The starting point of our work is a formulation of Banzhaf values in terms of a structured linear regression problem with an exponential number of rows.
We show in Theorem~\ref{thm:equivalence} that Banzhaf values are the exact solution to this linear regression problem.
Our efficient Kernel Banzhaf algorithm is then obtained by solving the regression problem approximately using just a small subset of rows. 

\subsection{Linear Regression Formulation}\label{sec:lr}

Let $\mathbf{A} \in \mathbb{R}^{2^n \times n}$ be a design matrix and $\mathbf{b} \in \mathbb{R}^{2^n}$ be a target vector.
We will use subsets $S \subseteq [n]$ to index the rows of $\mathbf{A}$ and entries of $\mathbf{b}$. In particular, $[\mathbf{A}]_{S, i}$ denotes the $i^\text{th}$ entry of row $S$. We set:
\begin{align}\label{eq:design_matrix}
    [\mathbf{A}]_{S, i} = \begin{cases}
        + \frac12 & \textnormal{ if $i \in S$} \\
        - \frac12 & \textnormal{ if $i \notin S$}
    \end{cases}
\end{align}
and
\begin{align}\label{eq:target_vector}
  \mathbf{b}_S = v(S).
\end{align}

\begin{restatable}[Linear Regression Equivalence]{theorem}{equivalence}\label{thm:equivalence}
    Consider $\mathbf{A}$ and $\mathbf{b}$ as defined in Equations \ref{eq:design_matrix} and \ref{eq:target_vector}, respectively. Let
    \begin{align*}
        \mathbf{x}^* = \argmin_\mathbf{x}
        \| \mathbf{A x-b} \|_2^2.
    \end{align*}
    Then $\boldsymbol{\phi} = \mathbf{x}^*$, where $\boldsymbol{\phi}$ are the Banzhaf values of $v$.
\end{restatable}

\begin{proof}[Proof of Theorem \ref{thm:equivalence}]
As is standard, we have that:
\begin{align*}
    \argmin_\mathbf{x}
    \| \mathbf{A x-b} \|_2^2
    = (\mathbf{A^\top A})^{-1} \mathbf{A}^\top \mathbf{b}.
\end{align*}
We will analyze the right hand side.
The $(i,j)$ entry in $\mathbf{A^\top A}$ is given by
\begin{align}\label{eq:ATA_ij}
    [\mathbf{A^\top A}]_{i,j} = \sum_{S \subseteq [n]}
     \left(\mathbbm{1}[i \in S]-\frac12\right)\left(\mathbbm{1}[j \in S]-\frac12\right)
\end{align}
If $i \neq j$, there are $2^{n-1}$ terms of $-\frac14$ when $\mathbbm{1}[i \in S]\neq \mathbbm{1}[j \in S]$ and $2^{n-1}$ terms of $+\frac14$ when $\mathbbm{1}[i \in S] = \mathbbm{1}[j \in S]$, hence Equation~\ref{eq:ATA_ij} is 0.
If $i=j$, we have $2^n$ terms of $+\frac14$ hence Equation~\ref{eq:ATA_ij} is $2^{n-2}$.
Together, this gives that 
\begin{align}\label{eq:ata}
 \mathbf{A^\top A} = 2^{n-2} \mathbf{I}.
\end{align}
Then $(\mathbf{A^\top A})^{-1} = \frac1{2^{n-2}} \mathbf{I}$.
Continuing, we have
\begin{align*}
    (\mathbf{A^\top A})^{-1} \mathbf{A^\top b}
    &= \frac1{2^{n-2}} \mathbf{A^\top b}
    \\&= \frac1{2^{n-2}} \sum_{S \subseteq[n]} \left(\mathbbm{1}[S]-\frac12 \mathbf{1}\right) v(S).
\end{align*}
We can write the $i^\text{th}$ entry of the above as:
\begin{align*}
    [(\mathbf{A^\top A})^{-1} \mathbf{A^\top b}]_i
    &= \frac1{2^{n-2}} \sum_{S \subseteq [n]} \left(\mathbbm{1}[i \in S]-\frac12\right) v(S)
    \\&= \frac1{2^{n-1}} \sum_{S \subseteq [n]\setminus \{i\}}
    v(S \cup \{i\}) - v(S)
\end{align*}
which is exactly Equation~\ref{eq:banzhaf_def}.
The statement follows.
\end{proof}

\paragraph{Other Regression Formulations.}

Despite the simplicity of the regression problem above, we are not aware of any prior work that considers it.
As discussed, a regression formulation for Banzhaf values is known for the special case where $v$ is binary and monotone i.e.,  $v:2^{[n]} \to \{0,1\}$ and $v(S) \leq v(T)$ for $S \subseteq T$ \cite{hammer1992approximations}.
Another relevant regression problem has the property that $\boldsymbol{\phi} = \mathbf{x}^* + c\mathbf{1}$ for $\mathbf{1} \in \mathbb{R}^n$ the all-ones vector and $c$ an unknown number that depends on $v$ \cite{ruiz1998family,li2024one}; however, it is not clear how to estimate $c$ efficiently and thus recover the Banzhaf values.

\subsection{The Kernel Banzhaf Algorithm}\label{sec:kernal_banzhaf}

Since $\mathbf{A}$ and $\mathbf{b}$ have $2^n$ rows, constructing the linear regression problem in Theorem \ref{thm:equivalence} to calculate the Banzhaf values is computationally prohibitive. In particular, we must evaluate $v(S)$ for all $2^n$ subsets of $[n]$ to compute $\mathbf{b}$.
Inspired by Kernel SHAP, we avoid this cost by constructing a much smaller regression problem that contains a subsample of $m \ll 2^n$ rows from $\mathbf{A}$ and $\mathbf{b}$.
Let the subsampled design matrix be $\tilde{\mathbf{A}} \in \mathbb{R}^{m \times n}$ and the subsampled target vector be $\tilde{\mathbf{b}} \in \mathbb{R}^m$.
The estimate we produce is 
    \vspace{-.5em}
\begin{align*}
    \approxphivec = \argmin_\mathbf{x} \| \tilde{\mathbf{A}} \mathbf{x} - \tilde{\mathbf{b}} \|_2^2.
\end{align*}
    \vspace{-1.75em}
    
The most direct approach to constructing $\tilde{\mathbf{A}}$ and $\tilde{\mathbf{b}}$ would be to subsample $m$ rows from $\mathbf{A}$ and $\mathbf{b}$ uniformly at random. We do so with a slight modification -- whenever the row corresponding to a set $S$ is sampled, we also select the row corresponding to $[n]\setminus S$. Referred to as ``paired sampling'', this approach has been used in the Kernel SHAP algorithm  for Shapley values \cite{covert2020improving}. While it has no impact on our theoretical results (unpaired sampling would give almost identical bounds) we observe an experimental improvement for Banzhaf value estimation as well. See e.g, \Cref{fig:l2-by-sample-size}, where unpaired sampling is denoted by ``Kernel Banzhaf (excl. Pairs).'' Detailed pseudocode for our Kernel Banzhaf appears in Algorithm~\ref{alg:kernel_banzhaf}.
In Appendix \ref{appendix:swor}, we also explore the possibility of uniform sampling \emph{without replacement}; ultimately, we find that sampling with replacement is simpler and offers similar performance,  until $m \geq 2^n$ when we can exactly compute Banzhaf values anyways.

\begin{algorithm}[tb]
   \caption{Kernel Banzhaf}
   \label{alg:kernel_banzhaf}
\begin{algorithmic}
   \STATE {\bfseries Input:} $n$: positive integer, $v: \{0,1\}^n \to \mathbb{R}$: set function, $m$ : even number of samples such that $n < m \leq 2^n$
   \STATE {\bfseries Output:} $\approxphivec \in \mathbb{R}^n$: estimated Banzhaf values
   \STATE Initialize $\tilde{\mathbf{A}} \gets \mathbf{0}_{m \times n}$
   \FOR{$j \in \{0,2,4, \ldots,m\}$}
      \STATE Sample $S_j$ uniformly from all subsets of $[n]$
      \STATE Set $\tilde{\mathbf{A}}_{j} \gets \mathbbm{1}[S_j] -\frac12 \mathbf{1}$
      \STATE Compute $S_{j+1} = [n] \setminus S$ \hfill // Paired sampling
      \STATE Set $\tilde{\mathbf{A}}_{j+1} \gets \mathbbm{1}[S_{j+1}] -\frac12 \mathbf{1}$
   \ENDFOR
   \STATE Compute $\tilde{\mathbf{b}} \gets [v(S_1), \ldots, v(S_m)]$
   \STATE Solve $\approxphivec \gets \argmin_{\mathbf{x}} \| \tilde{\mathbf{A}} \mathbf{x} - \tilde{\mathbf{b}}\|_2$ 
   \STATE {\bfseries Return:} $\approxphivec$
\end{algorithmic}
\end{algorithm}

\begin{table*}[t]
  \centering
  \resizebox{\linewidth}{!}{%
  \begin{tabular}{l ccc ccc ccc}
    \toprule
    & \multicolumn{3}{c}{MC} & \multicolumn{3}{c}{MSR} & \multicolumn{3}{c}{Kernel Banzhaf}\\
    \cmidrule(lr){2-4} \cmidrule(lr){5-7} \cmidrule(lr){8-10}
    \textbf{Dataset} & 25\% & Median & 75\% & 25\% & Median & 75\% & 25\% & Median & 75\%\\
    \midrule
    Diabetes ($n=8$) & \cellcolor{whitesilver!50}0.0053 & \cellcolor{whitesilver!50}0.0173 & \cellcolor{whitesilver!50}0.0357 & \cellcolor{whitebronze!50}0.0275 & \cellcolor{whitebronze!50}0.0368 & \cellcolor{whitebronze!50}0.0537 & \textbf{\cellcolor{whitegold!50}0.0002}& \textbf{\cellcolor{whitegold!50}0.0006} & \textbf{\cellcolor{whitegold!50}0.0011}\\
    Adult ($n=14$) & \cellcolor{whitesilver!50}0.0075 & \cellcolor{whitesilver!50}0.0116 & \cellcolor{whitesilver!50}0.0263 & \cellcolor{whitebronze!50}0.0387 & \cellcolor{whitebronze!50}0.0482 & \cellcolor{whitebronze!50}0.0633 & \textbf{\cellcolor{whitegold!50}0.0004} & \textbf{\cellcolor{whitegold!50}0.0006} & \textbf{\cellcolor{whitegold!50}0.0010}\\
    Bank ($n=16$) & \cellcolor{whitesilver!50}0.0106 & \cellcolor{whitesilver!50}0.0169 & \cellcolor{whitesilver!50}0.0448 & \cellcolor{whitebronze!50}0.0412 & \cellcolor{whitebronze!50}0.0512 & \cellcolor{whitebronze!50}0.0663 & \textbf{\cellcolor{whitegold!50}0.0009} & \textbf{\cellcolor{whitegold!50}0.0012} & \textbf{\cellcolor{whitegold!50}0.0022}\\
    German Credit ($n=20$) & \cellcolor{whitesilver!50}0.0085 & \cellcolor{whitesilver!50}0.0158 & \cellcolor{whitesilver!50}0.0352 & \cellcolor{whitebronze!50}0.0399 & \cellcolor{whitebronze!50}0.0511 & \cellcolor{whitebronze!50}0.0682 & \textbf{\cellcolor{whitegold!50}0.0006} & \textbf{\cellcolor{whitegold!50}0.0009} & \textbf{\cellcolor{whitegold!50}0.0015}\\
    NHANES ($n=79$) & \cellcolor{whitesilver!50}0.0008 & \cellcolor{whitesilver!50}0.0019 & \cellcolor{whitesilver!50}0.0038 & \cellcolor{whitebronze!50}0.0444 & \cellcolor{whitebronze!50}0.0496 & \cellcolor{whitebronze!50}0.0533 & \textbf{\cellcolor{whitegold!50}0.0000} & \textbf{\cellcolor{whitegold!50}0.0000} & \textbf{\cellcolor{whitegold!50}0.0000}\\
    BRCA ($n=100$) & \cellcolor{whitesilver!50}0.0057 & \cellcolor{whitesilver!50}0.0114 & \cellcolor{whitesilver!50}0.0415 & \cellcolor{whitebronze!50}0.0479 & \cellcolor{whitebronze!50}0.0533 & \cellcolor{whitebronze!50}0.0672 & \textbf{\cellcolor{whitegold!50}0.0002} & \textbf{\cellcolor{whitegold!50}0.0005} & \textbf{\cellcolor{whitegold!50}0.0017}\\
    Communities and Crime ($n=101$) & \cellcolor{whitesilver!50}0.0059 & \cellcolor{whitesilver!50}0.0123 & \cellcolor{whitesilver!50}0.0282 & \cellcolor{whitebronze!50}0.0484 & \cellcolor{whitebronze!50}0.0547 & \cellcolor{whitebronze!50}0.0602 & \textbf{\cellcolor{whitegold!50}0.0005} & \textbf{\cellcolor{whitegold!50}0.0008} & \textbf{\cellcolor{whitegold!50}0.0014}\\
    TUANDROMD ($n=241$) & \cellcolor{whitesilver!50}0.0050 & \cellcolor{whitesilver!50}0.0155 & \cellcolor{whitesilver!50}0.0431 & \cellcolor{whitebronze!50}0.0515 & \cellcolor{whitebronze!50}0.0553 & \cellcolor{whitebronze!50}0.0612 & \textbf{\cellcolor{whitegold!50}0.0002} & \textbf{\cellcolor{whitegold!50}0.0012} & \textbf{\cellcolor{whitegold!50}0.0022}\\
    \bottomrule
  \end{tabular}
  }
  \caption{The relative squared $\ell_2$-norm error (i.e., $\|\approxphivec - \phivec\|_2^2 / \|\phivec\|_2^2$) of Banzhaf estimators with $m=20n$ samples over 50 runs. We bold the lowest error for the 25\% percentile, median, and 75\% percentile.
  Across all eight datasets and every summary statistic, Kernel Banzhaf returns estimates with the lowest error, typically by an order of magnitude.
  }
  \label{tab:mc-msr-kb-20n}
  \vspace{-.3cm}
\end{table*}

\paragraph{Time Complexity.}
Let $T$ be the time complexity of evaluating the set function $v$. The total runtime of our Kernel Banzhaf method is $O(Tm + mn^2)$. In particular, it takes $O(n)$ time to sample a random set and its compliment, for a total of $O(nm)$ cost to construct $\tilde{\mathbf{A}}$. It then takes $O(Tm)$ time to construct $\tilde{\mathbf{b}}$, and finally $O(mn^2)$ time to solve the $m \times n$ least squared problem involving $\tilde{\mathbf{A}}$ and $\tilde{\mathbf{b}}$.
For most applications in ML and explainable AI, we expect the $O(Tm)$ term to dominate.
For example, even a forward pass on a shallow fully connected neural network with $n$ features takes $O(n^2)$ time, so evaluating any attribution-based value function takes at least $O(mn^2)$ time in total.
Figure \ref{fig:time} in Appendix \ref{app:time_complexity} confirms that the cost of Kernel Banzhaf (and prior Monte Carlo methods) is dominated by the number of samples, $m$, i.e., the number of evaluations of $v$.

\subsection{Approximation Guarantees}\label{sec:guarantees}
Despite its relative simplicity, it is possible to prove strong theoretically guarantees on the approximation quality of the Kernel Banzaf method. Such guarantees have recently been show for regression-based Shapley evaluation using results on subsampled regression problems from randomized numerical linear algebra \cite{musco2024leverage,woodruff2014sketching,drineas2018lectures}. In \Cref{app:theory_guarantees}, we use similar tools to first prove that, using a near-linear number of samples, we can solve the Banzhaf value regression formulation from \Cref{thm:equivalence} near optimally:
\begin{restatable}{theorem}{objectiveapprox}\label{thm:objective_approx}
    If $m = O(n \log \frac{n}{\delta} + \frac{n}{\delta \epsilon})$,
    Algorithm~\ref{alg:kernel_banzhaf} produces an estimate $\approxphivec$ that satisfies, with probability $1-\delta$,
    \begin{align}\label{eq:objective_approx}
        \| \mathbf{A} \approxphivec - \mathbf{b} \|_2^2 \leq (1+\epsilon) \| \mathbf{A} \phivec - \mathbf{b} \|_2^2.
    \end{align}
\end{restatable}
Due to a known lower bound on the sample complexity of agnostic active linear regression \cite{price_chen}, Theorem~\ref{thm:objective_approx} can only be improved in the logarithmic factor on $n$ \cite{price_chen}. For constant $\epsilon$ and $\delta$ this is intuitive: in particular, consider a linear set function where $\mathbf{b}$ is in the span of $\mathbf{A}$.
Then the right-hand side is 0 and we must recover $\phivec$ exactly.
To do this, we need to observe at least $n$ linearly independent rows of the regression problem.

As a direct consequence of \Cref{thm:objective_approx}, we obtain a bound on the Euclidean distances between the approximate Banzhaf values computed by our Kernel Banzhaf method and the true values. In particular, in Appendix \ref{app:theory_guarantees}, we show:
\begin{restatable}{corollary}{coronorm}\label{coro:l2norm}
    Let $\gamma = \| \mathbf{A} \boldsymbol{\phi} - \mathbf{b} \|_2^2 / \| \mathbf{A} \phivec \|_2^2$.
    Any $\approxphivec$ that satisfies Equation~\ref{eq:objective_approx} also satisfies
    \begin{align}
    \label{eq:ell_2_bound}
    \| \approxphivec - \phivec \|_2^2
    \leq \epsilon \gamma \| \phivec \|_2^2.
    \end{align}
    The reverse is also true: Equation~\ref{eq:ell_2_bound} implies Equation~\ref{eq:objective_approx}.
\end{restatable}

\begin{figure*}[t]
    \centering
    \includegraphics[width=\linewidth]{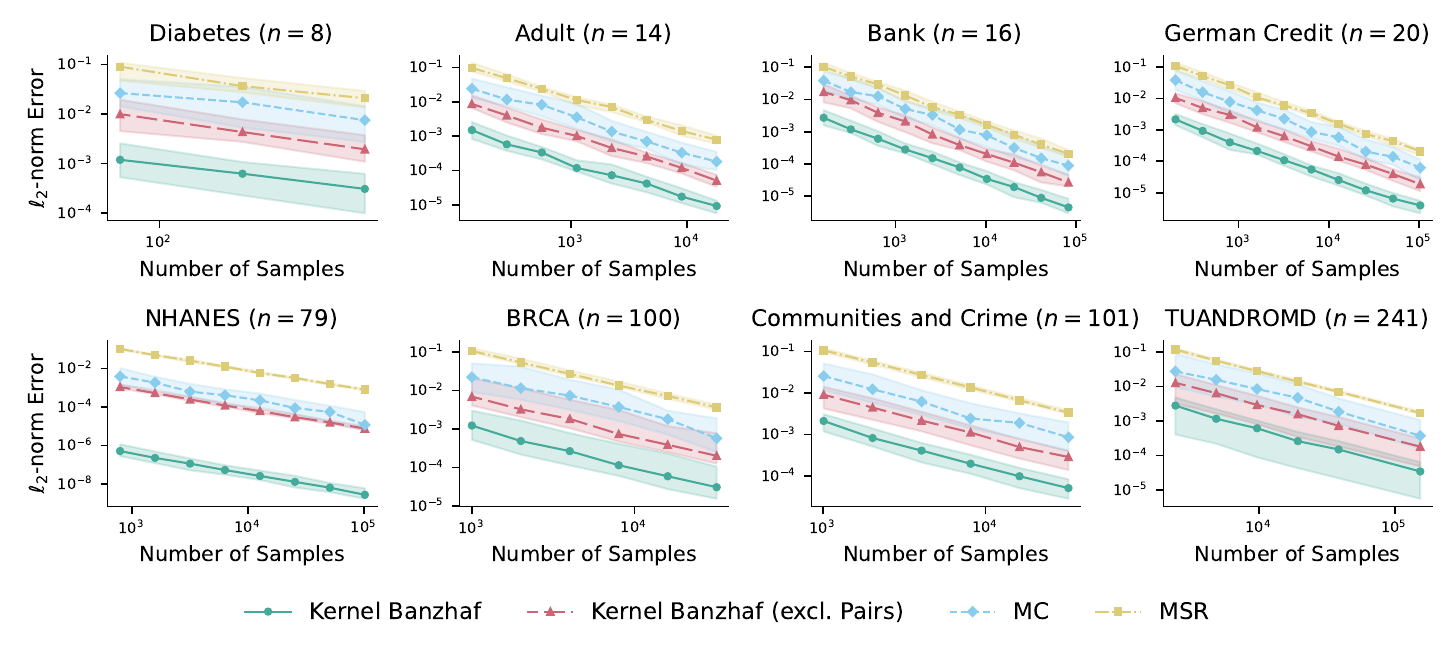}
    \caption{Plots comparing the relative squared $\ell_2$-norm error (i.e., $\|\approxphivec - \phivec\|_2^2 / \|\phivec\|_2^2$) for each Banzhaf estimator across increasing sample sizes in eight datasets. Each point represents the median of 50 runs, with shaded areas indicating the 25th to 75th percentiles.}
    \label{fig:l2-by-sample-size}
    \vspace{-.3cm}
\end{figure*}
Corollary~\ref{coro:l2norm} incurs a dependence on a parameter $\gamma$, which can be viewed as capturing how close $v$ is to an additive set function (i.e., how close $\mathbf{b}$ is to a linear function in the rows of $\mathbf{A}$, which can be viewed as set indicator vectors). Since there is no relaxation in the guarantee from Theorem~\ref{thm:objective_approx} to Corollary~\ref{coro:l2norm} and, as discussed, the sample complexity of Theorem~\ref{thm:objective_approx} is near optimal, we believe a dependence on $\gamma$ is inherent to the Kernal Banzhaf method. Fortunately, we find experimentally that the quantity is typically small -- the median value of $\gamma$ was around $100$ in our experiments on datasets that were small enough to compute the quantity. In an idea case, we could even have $\gamma = 0$ when $\mathbf{b}$ lies in the span of $\mathbf{A}$, and we would recover the exact Banzhaf values. 

\paragraph{Comparison of Theoretical Guarantees.}
Basic concentration bounds can be used to prove similar guarantees for the Monte Carlo and Maximum Sample Reuse estimators (see e.g., Theorems 4.8 and 4.9 in \citet{wang2023databanzhaf}).
When MC is run with $O\left(\frac{n^2}{\epsilon} \log (\frac{n}{\delta} )\right)$ samples and when MSR is run with $O\left(\frac{n}{\epsilon} \log (\frac{n}{\delta} )\right)$ samples, they respectively return estimates $\tilde{\boldsymbol{\phi}}$ that satisfy, with high probability,
\begin{align*}
    \| \approxphivec - \phivec \|_2^2
    \leq \epsilon \max_{S \subseteq [n]} v(S)^2.
\end{align*}
That is, both prior methods obtain a bound that depends on the maximum of $v$, rather than $\| \boldsymbol{\phi}\|_2^2$.
In general, we expect the maximum magnitude of the set function to be larger than the Banzhaf values, which measure an \textit{average} marginal contribution. This is reflected in the superior Kernel Banzhaf in our experiments (see e.g., Table \ref{tab:mc-msr-kb-20n}).

\subsection{Extension to Probabilistic Values}
Banzhaf and Shapley values can be generalized to so-called \textit{probabilistic values}, which have also found applications in machine learning \cite{kwon22betashapley,li2024one,li2024robust}.
In this section, we show how our regression formulation can be extended to probabilistic values.

For a weight vector $\mathbf{p} \in [0,1]^n$ such that $\sum_{\ell=0}^{n-1} \binom{n-1}{\ell} p_{\ell} = 1$, the \emph{probabilistic value} is given by
\begin{align*}
    \phi^{\textnormal{prob}}_i = \sum_{S \subseteq [n] \setminus \{i\}} p_{|S|}
    [v(S \cup \{i\}) - v(S)].
\end{align*}
Define the quantities $a_n$ and $b_n$ as follows:
\begin{align}\label{eq:an}
    a_n &\vcentcolon= 2 \sum_{\ell=0}^{n-1} \binom{n-1}{\ell} p_{\ell}^2
    - \sum_{\ell=1}^{n-1} \binom{n-2}{\ell-1} \left(p_\ell - p_{\ell-1}\right)^2\\
\label{eq:bn}
    b_n &\vcentcolon= \frac1{a_n} \sum_{\ell=1}^{n-1} \binom{n-2}{\ell-1} \left(p_\ell - p_{\ell-1}\right)^2.
\end{align}
Note that both $a_n$ and $b_n$ depend only on $n$ and $\mathbf{p}$, and can be computed in $O(n)$ time.
In the Banzhaf setting when $p_{\ell} = p_{\ell-1}$ for all $\ell$, observe that $b_n=0$.

Construct a design matrix be $\mathbf{A}^{\textnormal{prob}} \in \mathbb{R}^{2^n \times n}$ with
\begin{align*}
    [\mathbf{A}^{\textnormal{prob}} ]_{S,i} = \begin{cases}
        \frac{p_{|S|-1}}{a_n} & \textrm{if } i \in S \\
        -\frac{p_{|S|}}{a_n} & \textrm{if } i \notin S.
    \end{cases}
\end{align*}
Exactly as we did for Banzhaf values, let $\mathbf{b}^{\textnormal{prob}}_S = v(S)$.

\begin{restatable}[Extended Regression Equivalence]{lemma}{unconstrainedreg}\label{lemma:unconstrained_regression}
Let
\begin{align*}
  \mathbf{x}^*
    =\argmin_{\mathbf{x} \in \mathbb{R}^n} \| \mathbf{A}^{\textnormal{prob}} \mathbf{x} - \mathbf{b}^{\textnormal{prob}} \|_2^2.
\end{align*}
Then $\boldsymbol{\phi}^\textnormal{prob} = (\mathbf{I} + b_n \mathbf{1 1^\top}) \mathbf{x}^{*}$.
\end{restatable}

We defer the proof of Lemma \ref{lemma:unconstrained_regression} to Appendix \ref{app:prob_vals}.

Like in the Banzhaf setting, we can sample the full regression problem to produce a subsampled design matrix $\tilde{\mathbf{A}}^\textnormal{prob} \in \mathbb{R}^{m \times n}$ and  target vector $\tilde{\mathbf{b}}^\textnormal{prob} \in \mathbb{R}^{m}$.
Let ${\tilde{\boldsymbol{\phi}}}^\textnormal{prob} = \argmin_{\mathbf{x}} \| \tilde{\mathbf{A}}^\textnormal{prob} \mathbf{x} - \tilde{\mathbf{b}}^\textnormal{prob}\|_2^2$ be the solution to this subsampled problem.
The following theorem (proven in  Appendix \ref{app:prob_vals}) gives guarantees on the accuracy of this solution.

\begin{figure*}
    \centering
    \includegraphics[width=\linewidth]{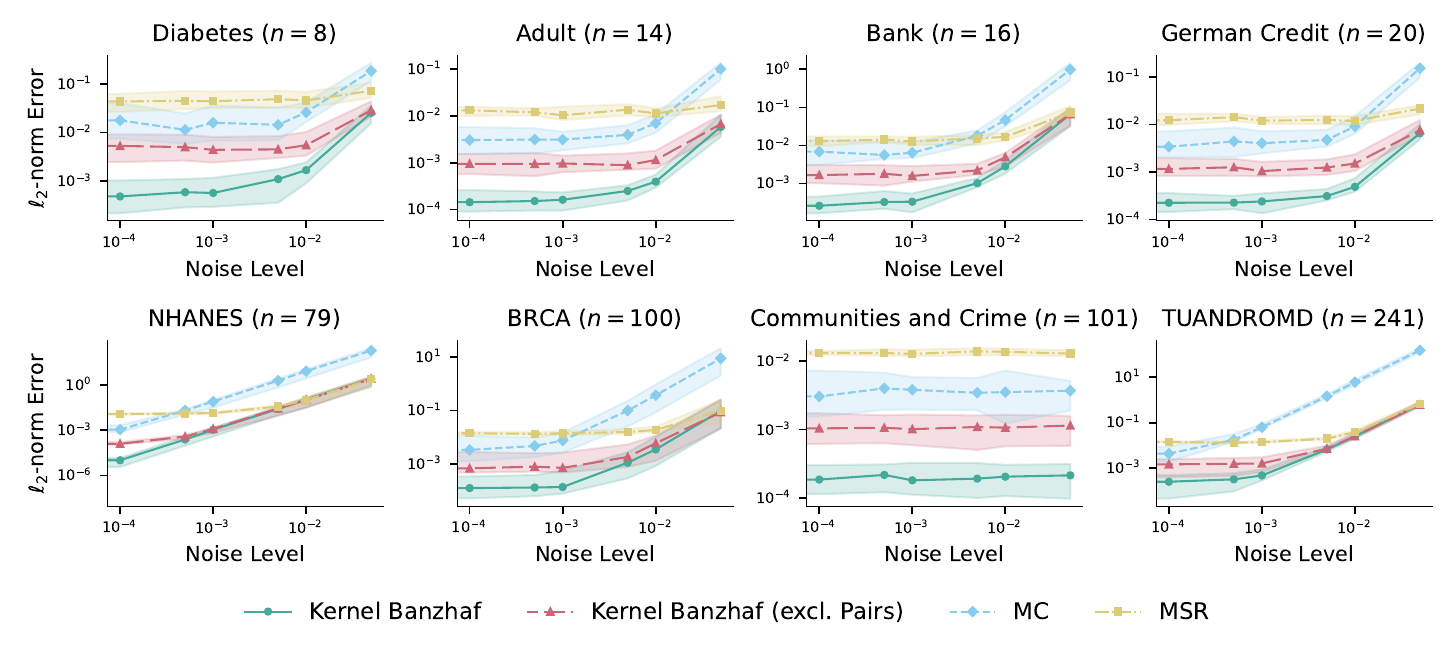}
    \caption{Plots of relative squared $\ell_2$-norm error by noise level across Banzhaf estimators. For each noise level $\sigma$, the estimator observes $v(S) + x$ where $x \sim \mathcal{N}(0,\sigma)$.
    Kernel Banzhaf outperforms for lower noise levels, eventually matching its ablated version and MSR for larger noise. MC is worse for all noise settings, likely because its constituent estimates are $v(S \cup \{i \}) + x - v(S) - x'$, increasing the variance of the noise.
    }
    \label{fig:l2-by-noise}
    \vspace{-.3cm}
\end{figure*}

\begin{restatable}[Probabilistic Value Approximation]{theorem}{probguarantee}\label{thm:prob_guarantee}
    Let $\gamma = \|\mathbf{A}^\textnormal{prob} \mathbf{x}^* - \mathbf{b}^\textnormal{prob}\|_2^2/\| \mathbf{A}^\textnormal{prob} \mathbf{x}^*\|_2^2$. There is a sampling method that takes $m = O(n \log (n/\delta) + n \log(1/\delta) / \epsilon)$ samples and ensures that, with probability $1-\delta$,
    \begin{align*}
        \| \boldsymbol{\phi}^\textnormal{prob} - {\tilde{\boldsymbol{\phi}}}^\textnormal{prob} \|_2^2
        \leq \epsilon \gamma (1+nb_n)^3 \| \boldsymbol{\phi}^\textnormal{prob} \|_2^2.
    \end{align*}
\end{restatable}
Note that Theorem \ref{thm:prob_guarantee} recovers Corollary \ref{coro:l2norm} for Banzhaf values since $b_n=0$.
Initial experiments suggest that the dependence on $b_n$ can limit the accuracy of this regression formualation for other probabilistic values, since we could have $b_n \gg 0$. However, we leave further exploration of the topic to future work. 

\section{Experiments}\label{sec:experiments}
We conduct extensive experimental evaluations of Banzhaf estimators across eight popular datasets.
\footnote{Our code is available at \href{https://github.com/lyrain2001/kernel-banzhaf}{https://github.com/lyrain2001/kernel-banzhaf}}.
In these experiments, we focus on feature attribution tasks for two reasons:
First, we are interested in explainable AI applications where users seek to interpret model predictions in terms of their input features.
Second, we can efficiently compute Banzhaf values even for large $n$, provided the set function corresponds to the feature attribution function of a decision tree model \cite{karczmarz2022improved}. This allows us to obtain a baseline, and thus an accurate calculation of the accuracy of various Banzhaf values estimation methods.

In addition to experiments estimating the Banzhaf values of tree-based models, we also explore neural network models using smaller datasets for which Equation~\ref{eq:banzhaf_def} can be used to calculate exact Banzhaf values. These neural network experiments are presented in Figure \ref{fig:l2-by-sample-size-nn} of Appendix~\ref{app:neural_network}, and align with our findings for decision trees.

We measure approximation error with $\|\approxphivec - \phivec\|_2^2 / \|\phivec\|_2^2$ i.e., the relative squared $\ell_2$-norm error between the estimated and true Banzhaf values. In Appendix~\ref{app:objective_error}, we also evaluate estimators in terms of the objective value $\| \mathbf{A \approxphivec-b} \|_2$ suggested by the linear regression formulation; the results in Figure \ref{fig:objective} corroborate our findings on the $\ell_2$-norm error.

We provide more details on the datasets and implementation details in Appendix~\ref{appendix:dataset}.


\begin{figure*}[h]
    \centering
    \includegraphics[width=\linewidth]{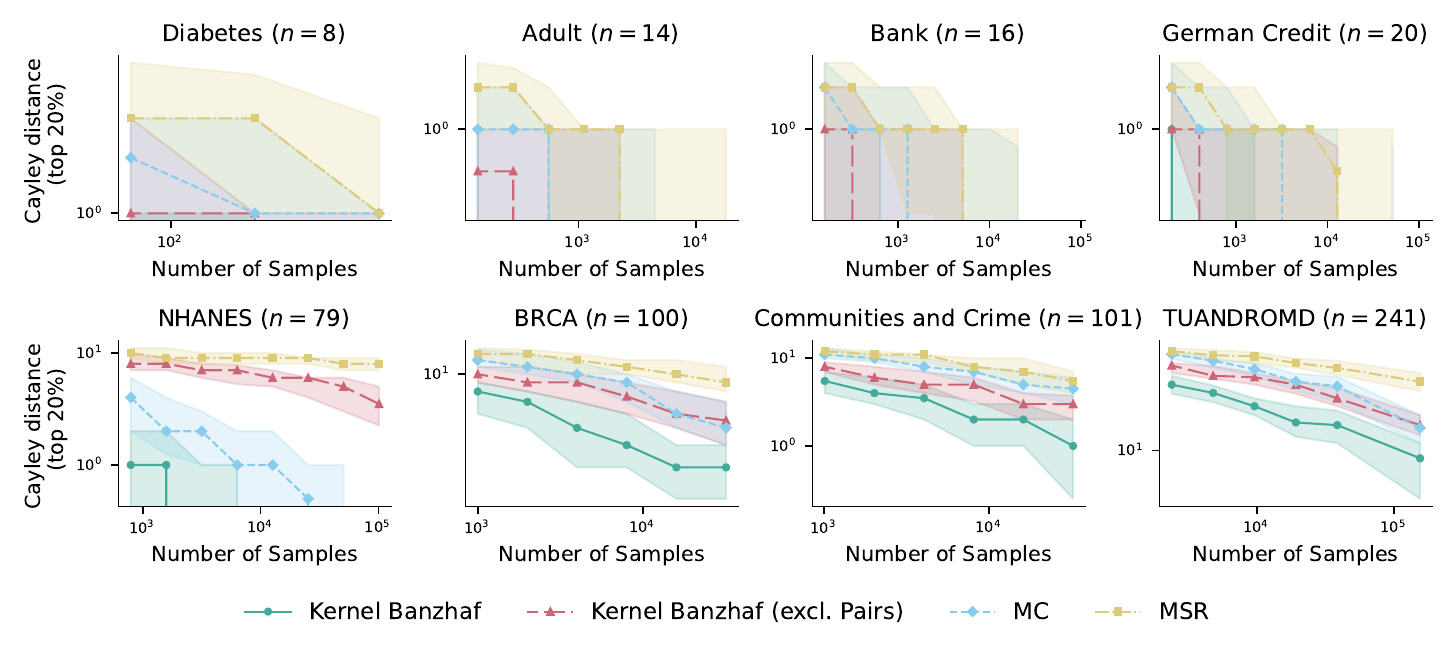}
    \caption{Comparison of top 20\% feature ranking recovery using Cayley distance (lower is more accurate). Recovering the most impactful features is an primary component of comparing and selecting the most important features for machine learnign models. Especially as the sizes of the datasets increase, Kernel Banzhaf gives the best performance.}
    \label{fig:topk_cayley}
    \vspace{-.3cm}
\end{figure*}

\subsection{Baselines} \label{sec:comp_ban_est}
We compare Kernel Banzhaf with the following methods:

\textbf{MC} The Monte Carlo (MC) algorithm estimates each Banzhaf value individually. Let $\mathcal{S}_i$ be the subsets sampled for the estimate of the $i$th Banzhaf value.
We have that $\sum_{i=1}^m |\mathcal{S}_i| = m$.
The MC estimate for player $i$ is given by
$$
    \tilde{\phi}^{\textrm{MC}}_i
    = \frac1{|\mathcal{S}_i|} \sum_{S \in \mathcal{S}_i} [v(S \cup \{i\}) - v(S)].
$$

\textbf{MSR} The Maximum Sample Reuse (MSR) algorithm estimates all Banzhaf values simultaneously \citep{wang2023databanzhaf}.
Let $\mathcal{S}$ be the subsets sampled for all players i.e., $|\mathcal{S}| = m$.
Define $\mathcal{S}_{\ni i}$ as the sampled subsets that contain player $i$ and $\mathcal{S}_{\not\owns i}$ as the sampled subsets that do not contain player $i$.
The MSR estimate for player $i$ is given by
$$
    \tilde{\phi}^{\textrm{MSR}}_i
    = \frac1{|\mathcal{S}_{\ni i}|} \sum_{S \in \mathcal{S}_{\ni i}} v(S) 
    - \frac1{|\mathcal{S}_{\not\owns i}|} \sum_{S \in \mathcal{S}_{\not\owns i}} v(S).
$$
While the MSR algorithm reuses samples, it can have high variance because the magnitude of the set function $v(S)$ is generally much larger than the marginal difference between nearby values 
$v(S \cup \{i\}) - v(S)$.

\textbf{Kernel Banzhaf (Excluding Pairs)} This algorithm is Kernel Banzhaf (Algorithm~\ref{alg:kernel_banzhaf}) but without paired sampling.

\subsection{Sample Efficiency}
Since evaluating the set function $v$ generally dominates the estimator's time complexity (see confirmation in Appendix \ref{app:time_complexity}), we investigate estimator error by number of samples in Figure~\ref{fig:l2-by-sample-size}.
As expected, the error for all estimators decreases as the sample size increases.
In terms of performance, Kernel Banzhaf gives the lowest error, followed by its ablated variant without paired sampling, and then MC.
In \citet{wang2023databanzhaf}, the MSR algorithm was shown to converge faster than the MC algorithm, but faster convergence does not always imply better accuracy.

\subsection{Robustness to Noise}
The set functions in explainable AI tasks are often estimates of a random process e.g., the expectation of a model's prediction on a randomly sampled observation.
We measure how robust each estimator is to inaccurate set function values by running experiments on the modified set function $v(S) + x$ where $x \sim \mathcal{N}(0,\sigma)$ for a given noise level $\sigma$.
As shown in Figure~\ref{fig:l2-by-noise}, Kernel Banzhaf  and its variant without paired sampling are more robust to this noise than the other algorithms, especially when the level of noise size is low.
This suggests that Kernel Banzhaf is particularly well-suited for real-world applications where the set function is approximated.
In Appendix \ref{app:adversarial_perturbations}, we explore how adversarial noise affects estimator performance, finding that Kernel Banzhaf continues to give the best performance in the adversarial noise setting.

\subsection{Feature Ranking Recovery}
Aside from evaluating the quantitative errors between estimated and exact Banzhaf values, we also evaluate how well each estimator recovers the rank of the most important features.
Such ranking tasks, where we are interested in recovering the rank rather than quantitative value of a feature, is a principal component of comparing and selecting important features for machine learning models.
Figure~\ref{fig:topk_cayley} shows the Cayley distance between the estimated ranks of the top 20\% of features and the true ranks.
Especially for datasets with larger $n$, Kernel Banzhaf gives the most accurate rankings.
We corroborate our findings in terms of the Spearman correlation coefficient (where a higher value is better) in Figure \ref{fig:topk_spearman} of Appendix \ref{app:feature_ranking}.

\section{Conclusion}

In this work, we present Kernel Banzhaf, the first regression-based estimator for Banzhaf values.
Our experiments suggest that Kernel Banzhaf significantly outperforms existing methods for estimating Banzhaf values. 
Our theoretical analysis establishes robust approximation guarantees that depend on the size of the Banzhaf values, in contrast to Monte Carlo estimators which depend on the magnitude of the underlying set function.

One avenue for future work is to evaluate the generalization of Kernel Banzhaf on other probabilistic values. However, it is not clear how to efficiently compute general probabilistic values and the poor approximation guarantee when $b_n \neq 0$ may be reflected in the generalized algorithm's performance.

\section*{Impact Statement}
This paper presents a new algorithm for more efficiently computing scores that are widely used to better understand decisions made by black-box machine learning algorithms. While all possible applications of the work are difficult to predict, we thus expect the results to have a net positive impact on the ethical and responsible application of machine learning in society.

\section*{Acknowledgements}
Freire and Liu were partially supported by NSF Awards IIS-2106888 and the DARPA ASKEM and ARPA-H BDF programs. Liu initiated this project during her summer internship at Google Research. Witter was supported by the National Science Foundation under Graduate Research Fellowship Grant No. DGE-2234660. Musco was partially supported by NSF Award CCF-2045590. Opinions, findings, conclusions, or recommendations expressed in this material are those of the authors and do not reflect the views of NSF or DARPA.

\bibliography{references}
\bibliographystyle{icml2025}

\newpage
\appendix
\onecolumn

\clearpage
\section{Kernel Banzhaf Approximation Guarantees}\label{app:theory_guarantees}

In this section, we prove theoretical guarantees on the performance of Kernel Banzhaf.
We begin with a standard regression sampling result, the proof of which we defer to the end of this section.

\begin{theorem}\label{thm:full_leverage_score}
    Consider a subsampling matrix $\mathbf{S} \in \mathbb{R}^{m \times 2^n}$ where each row contains exactly one non-zero entry, selected (possibly in a paired way) with probability proportional to the leverage scores of the corresponding row in $\mathbf{A}$.
    Define $\tilde{\mathbf{A}} = \mathbf{SA}$ and $\tilde{\mathbf{b}} = \mathbf{Sb}$.
    Let the solution to the subsampled problem be $\tilde{\boldsymbol{\phi}} = \argmin_\mathbf{x} \| \tilde{\mathbf{A}} \mathbf{x} - \tilde{\mathbf{b}} \|_2^2$.
    If $m = O(n \log (\frac{n}{\delta}) + \frac{n}{\delta \epsilon})$, then 
    \begin{align}
        \| \mathbf{A} \tilde{\boldsymbol{\phi}} - \mathbf{b} \|_2^2 \leq (1+\epsilon) \| \mathbf{A} \boldsymbol{\phi} - \mathbf{b}\|_2^2
    \end{align}
    with probability $1-\delta$.
\end{theorem}

In order to prove Theorem~\ref{thm:objective_approx}, we will show that uniform sampling is actually equivalent to leverage score sampling. The leverage scores of a matrix are an important statistical quantity that roughly measures the ``uniqueness'' of a row \citep{woodruff2014sketching,drineas2018lectures}.

The leverage scores of $\mathbf{A}$ are given by
\begin{align}
    \ell_S = [\mathbf{A}]_S^\top
    \left( \mathbf{A^\top A} \right)^{-1} [\mathbf{A}]_S
    = \frac1{2^{n-2}}
    \left(\mathbbm{1}[S] - \frac12 \right)^\top \left(\mathbbm{1}[S] - \frac12 \right)
    = \frac{n}{2^n}
\end{align}
where the second equality follows by Equation~\ref{eq:ata}.
Notice the leverage scores are equivalent to uniform sampling and Theorem \ref{thm:objective_approx} immediately follows.

We can now prove Corollary~\ref{coro:l2norm}.

\coronorm*

\begin{proof}[Proof of Corollary~\ref{coro:l2norm}]

We have
\begin{align}
    \| \mathbf{A} \approxphivec - \mathbf{b} \|_2^2
    = \| \mathbf{A} \approxphivec - \mathbf{A} \boldsymbol{\phi} + \mathbf{A} \boldsymbol{\phi} - \mathbf{b} \|_2^2    
    = \| \mathbf{A} \approxphivec - \mathbf{A} \boldsymbol{\phi} \|_2^2 + \| \mathbf{A} \boldsymbol{\phi} - \mathbf{b} \|_2^2
\end{align}
where the second equality follows because $\mathbf{A} \boldsymbol{\phi} - \mathbf{b}$ is orthogonal to any vector in the span of $\mathbf{A}$.
Then, by the assumption that $\| \mathbf{A} \tilde{\boldsymbol{\phi}} - \mathbf{b} \|_2^2 \leq (1+\epsilon) \| \mathbf{A} \boldsymbol{\phi}- \mathbf{b}\|_2^2$, we have $\| \mathbf{A} \approxphivec - \mathbf{A} \boldsymbol{\phi} \|_2^2 \leq \epsilon \| \mathbf{A} \boldsymbol{\phi} - \mathbf{b} \|_2^2$.
By Equation~\ref{eq:ata}, we have 
\begin{align}
  \|\mathbf{A} \boldsymbol{\phi} \|_2^2 = \boldsymbol{\phi}^\top \mathbf{A^\top A} \boldsymbol{\phi} = 2^{n-2} \| \boldsymbol{\phi} \|_2^2  
\end{align}
and, similarly, $\|\mathbf{A} (\approxphivec - \boldsymbol{\phi}) \ \|_2^2 = 2^{n-2}\|\approxphivec - \boldsymbol{\phi} \ \|_2^2$.
Then, with the definition of $\gamma$,
\begin{align}
    2^{n-2} \| \approxphivec - \boldsymbol{\phi} \|_2^2 =
    \| \mathbf{A}( \approxphivec - \boldsymbol{\phi}) \|_2^2 
    \leq \epsilon \| \mathbf{A} \boldsymbol{\phi} - \mathbf{b} \|_2^2
    = \epsilon \gamma \| \mathbf{A} \boldsymbol{\phi} \|_2^2
    = 2^{n-2} \epsilon \gamma \| \boldsymbol{\phi} \|_2^2.
\end{align}
The statement then follows after dividing both sides by $2^{n-2}$.
\end{proof}

Next, we prove Theorem~\ref{thm:full_leverage_score}, which is a standard guarantee for leverage score sampling.
However, because rows are sampled in \textit{pairs}, we need to modify the analysis.
In particular, both the spectral guarantee that the sampling matrix preserves eigenvalues and the Frobenius guarantee that the sampling matrix preserves Frobenius norm need to be reproved.

We will adopt the notation from \citet{wu2018note}.
Let's consider a leverage score sampling method where rows are selected in blocks.
Define $\Theta$ as a partition of blocks, each containing 2 elements with identical leverage scores in our setting. 
We assign a sampling probability $p^+_i$ to each block $\Theta_i$, calculated as the sum of leverage scores in that block divided by the total sum of all leverage scores:
$p^+i \vcentcolon= \frac{\sum{k \in \Theta_i} \ell_k}{\sum_{j} \ell_j}$.
For simplicity of notation, suppose $m$ is even.
Let $\mathbf{S} \in \mathbb{R}^{m \times \rho}$ be a sampling matrix, initially set to $\mathbf{0}_{m \times n}$.
The sampling process repeats $m/2$ times:
Sample a block $\Theta_{i}$ with probability $p^+_{i}$.
For each $k \in \Theta_i$, set the $k$th entry in an empty row to $\frac{1}{\sqrt{p^+_i}}$.

To analyze the solution obtained from this block-wise leverage score sampling, we will demonstrate that the sampling matrix $\mathbf{S}$ preserves both the spectral norm and the Frobenius norm.

\begin{lemma}[Spectral Approximation]\label{lemma:spectral}
    Let $\mathbf{U} \in \mathbb{R}^{\rho \times n}$ be a matrix with orthonormal columns.
    Consider the block random sampling matrix $\mathbf{S}$ described above with rows sampled according to the leverage scores of $\mathbf{U}$.
    When $m = \Omega(n \log (n/\delta)/\epsilon^2)$,
    \begin{align}
        \| \mathbf{I} - \mathbf{U^\top S^\top S U} \|_2 \leq \epsilon
    \end{align}
    with probability $1-\delta$.
\end{lemma}

\begin{proof}[Proof of Lemma~\ref{lemma:spectral}]
We will use the following matrix Chernoff bound (see e.g., Fact 1 in \citet{woodruff2014sketching}).

\begin{fact}[Matrix Chernoff]\label{fact:matrix_chernoff}
Let $\mathbf{X}_1, \ldots, \mathbf{X}_m \in \mathbb{R}^{n \times n}$ be independent samples of symmetric random matrices with $\mathbb{E}[\mathbf{X}_j] = \mathbf{0}$, $\| \mathbf{X}_j \|_2 \leq \gamma$ for all $j$, and $\| \mathbb{E}_j[\mathbf{X}_j^2] \|_2 \leq \sigma^2$.
Then for any $\epsilon > 0$,
\begin{align}
    \Pr\left(\left\| \frac1{m} \sum_{j=1}^m \mathbf{X}_j \right\|_2 \geq \epsilon\right) \leq 2 n \exp\left(\frac{-m \epsilon^2}{\sigma^2 + \gamma \epsilon/3}\right).
\end{align}
\end{fact}

For sample $j \in [m]$, let $i(j)$ be the index of the block selected.
Define 
\begin{align}
    \mathbf{X}_j = \mathbf{I} 
    - \frac1{p^+_{i(j)}} \sum_{k \in \Theta_{i(j)}} \mathbf{U}_k ^\top \mathbf{U}_k
\end{align}
We will compute $\mathbf{E}[X_j]$, $\| \mathbf{X}_j \|_2$, and $\| \mathbb{E}[\mathbf{X}_j^2] \|_2$.
First,
\begin{align}
    \E[\mathbf{X}_j]
    = \mathbf{I} 
    -  \sum_{i =1}^{| \Theta|} p^+_{i} \frac1{p^+_{i}}
    \sum_{k \in \Theta_{i}} \mathbf{U}_k ^\top \mathbf{U}_k
    = \mathbf{0}
\end{align}
where the last equality follows because $\Theta$ is a partition and $\mathbf{U}^\top \mathbf{U} = \mathbf{I}$.
Next, note that
\begin{align}
    \| \mathbf{X}_j \|_2 \leq \|\mathbf{I} \|_2 +
    \frac{\sum_{k \in \Theta_{i(j)}} \|\mathbf{U}_k^\top \mathbf{U}_k\|_2}
    {\sum_{k \in \Theta_{i(j)}} p_k}
    = 1 + n \frac{\sum_{k \in \Theta_{i(j)}} \|\mathbf{U}_k \|_2^2}
    {\sum_{k \in \Theta_{i(j)}} \ell_k} = 1 + n
\end{align}
where the last equality follows because $\| \mathbf{U}_k \|_2^2 = \ell_k$ since $\mathbf{U}^\top \mathbf{U} = \mathbf{I}$.
Define $\mathbf{U}_{\Theta_i} \in \mathbb{R}^{|\Theta_i| \times n}$ as the matrix with rows $\mathbf{U}_k$ for $k \in \Theta_i$. 
Observe that $\sum_{k \in \Theta_i} \mathbf{U}_k^\top \mathbf{U}_k = \mathbf{U}_{\Theta_i}^\top\mathbf{U}_{\Theta_i}$.
Finally, note that
\begin{align}
    \E[\mathbf{X}_j^2]
    &= \mathbf{I} - 2 \sum_{i=1}^{|\Theta|} p^+_i
    \frac{\sum_{k \in \Theta_i} \mathbf{U}_k^\top \mathbf{U}_k}
    {p^+_i}
    + \sum_{i=1}^{|\Theta|}
    p^+_i \frac{\left(\mathbf{U}_{\Theta_i}^\top\mathbf{U}_{\Theta_i}\right)^2}{{p^+_{i}}^2}
    \\&= -\mathbf{I} + 
    \sum_{i=1}^{|\Theta|} \frac1{p^+_{i}} \mathbf{U}_{\Theta_i}^\top\mathbf{U}_{\Theta_i}\mathbf{U}_{\Theta_i}^\top\mathbf{U}_{\Theta_i}.
    \label{eq:Exj2}
\end{align}
Observe that entry $(k, k')$ of $\mathbf{U}_{\Theta_i}\mathbf{U}_{\Theta_i}^\top \in \mathbb{R}^{|\Theta_i| \times |\Theta_i|}$ is $\mathbf{U}_k \mathbf{U}_{k'}^\top$.
So the absolute value of each entry is $|\mathbf{U}_k \mathbf{U}_{k'}^\top| \leq \| \mathbf{U}_k \|_2 \| \mathbf{U}_{k'} \|_2 = \ell_k^{1/2} \ell_{k'}^{1/2}$ by Cauchy-Schwarz.
Define $\ell^{\max}_i = \max_{k \in \Theta_i} \ell_k$ and $\ell^{\min}_i = \min_{k \in \Theta_i} \ell_k$.
By the Gershgorin circle theorem, $\mathbf{U}_{\Theta_i}\mathbf{U}_{\Theta_i}^\top \preceq 2\ell^{\max}_i |\Theta_i|^2 \mathbf{I}$.
Equivalently, $\mathbf{A} \preceq \mathbf{B}$, $\mathbf{x}^\top \mathbf{A} \mathbf{x} \leq \mathbf{x}^\top \mathbf{B} \mathbf{x}$ for all $\mathbf{x}$.
Consider an arbitrary $\mathbf{z}$.
We have $\mathbf{z^\top C^\top A C z} \leq \mathbf{z^\top C^\top B C z}$ since $\mathbf{Cz}$ is some $\mathbf{x}$.
It follows that $\mathbf{U}_{\Theta_i}^\top\mathbf{U}_{\Theta_i}\mathbf{U}_{\Theta_i}^\top\mathbf{U}_{\Theta_i} \preceq 2\ell^{\max}_i |\Theta_i|^2 \mathbf{U}_{\Theta_i}^\top\mathbf{U}_{\Theta_i}$.
Then
\begin{align}
    (\ref{eq:Exj2})
    \preceq n \sum_{i=1}^{|\Theta_i|}
    \frac{2\ell^{\max}_i |\Theta_i|^2 \mathbf{U}_{\Theta_i}^\top\mathbf{U}_{\Theta_i}}
    {\sum_{k\in \Theta_i} \ell_k}
    \preceq n \max_{i} 2 |\Theta_i| \frac{\ell^{\max}_i}{\ell^{\min}_i} \mathbf{I}.
\end{align}
Since $|\Theta_i| \leq 2$ and the leverage scores in a block are all equal, $\| \E[\mathbf{X}_j^2] \|_2 \leq 4n$.

Applying Fact~\ref{fact:matrix_chernoff} with $m = O(n \log(n/\delta)/\epsilon^2)$
yields
\begin{align}
    \Pr \left( \left\| \frac1{m} \sum_{j=1}^m \mathbf{I} 
    - \frac1{p^+_{i(j)}} \mathbf{U}_{\Theta_i}^\top \mathbf{U}_{\Theta_i} \right \|_2 \geq \epsilon \right) \leq \delta.
\end{align}
The lemma statement follows.


\end{proof}


We will also show that the sampling matrix preserves the Frobenius norm.

\begin{lemma}[Frobenius Approximation]\label{lemma:frobenius}
    Consider the block random sampling matrix $\mathbf{S}$ described above with rows sampled according to the leverage scores of $\mathbf{U} \in \mathbb{R}^{\rho \times n}$.
    Let $\mathbf{V} \in \mathbb{R}^{\rho \times n'}$.
    As long as $m \geq \frac1{\delta \epsilon^2}$, then
    \begin{align}
        \left\| \mathbf{U^\top S^\top S V} - \mathbf{U^\top V} \right\|_F \leq \epsilon \| \mathbf{U}\|_F \|\mathbf{V} \|_F
    \end{align}
    with probability $1-\delta$.
\end{lemma}

\begin{proof}[Proof of Lemma \ref{lemma:frobenius}]
By Proposition 2.2 in \citet{wu2018note}, we have that
\begin{align}
    \mathbb{E} [ \|  \mathbf{U^\top S^\top S V} - \mathbf{U^\top V}\|_F^2]
    \leq \frac1{m} \sum_{j=1}^{|\Theta|}
    \frac1{p^+_{i(j)}} \|\mathbf{U}_{\Theta_{i(j)}}\|_F^2 \|\mathbf{V}_{\Theta_{i(j)}}\|_F^2
\end{align}
where $\mathbf{U}_{\Theta_i} \in \mathbb{R}^{n \times |\Theta_i|}$ is a matrix with columns $\mathbf{U}_k$ for $k \in \Theta_i$.
Because $p_k = \| \mathbf{U}_k \|_2^2 / \| \mathbf{U} \|_F^2$ by the definition of leverage scores, we have
\begin{align}
    \mathbb{E} [ \|  \mathbf{U^\top S^\top S V} - \mathbf{U^\top V} \|_F^2]
    \leq \frac1{m} \sum_{i=1}^{|\Theta|}
    \| \mathbf{U} \|_F^2 \frac{\sum_{k \in \Theta_i} \| \mathbf{U}_k \|_2^2}{\sum_{k \in \Theta_i} \| \mathbf{U}_k \|_2^2} \|\mathbf{V}_{\Theta_i}\|_F^2
    = \frac1{m} \| \mathbf{U} \|_F^2 \| \mathbf{V} \|_F^2.
\end{align}
By Markov's inequality,
\begin{align}
    \Pr \left(\left\| \mathbf{U^\top S^\top S V} - \mathbf{U^\top V} \right\|_F > \epsilon \| \mathbf{U}\|_F \|\mathbf{V} \|_F \right)
    &\leq \frac{\E \left[ \left\| \mathbf{U^\top S^\top S V} - \mathbf{U^\top V} \right\|_F^2 \right]}{\epsilon^2 \| \mathbf{U}\|_F^2 \|\mathbf{V} \|_F^2}
    \leq \frac1{m \epsilon^2} \leq \delta
\end{align}
as long as $m \geq \frac1{\delta \epsilon^2}$.
\end{proof}

With Lemmas~\ref{lemma:spectral} and~\ref{lemma:frobenius} already proved for the special paired leverage score sampling, the following analysis is standard.
We include the following proof for completeness.

\begin{proof}[Proof of Theorem~\ref{thm:full_leverage_score}]
Observe that
\begin{align}\label{eq:to_prove_paired}
    \| \mathbf{A} \tilde{\boldsymbol{\phi}} - \mathbf{b} \|_2^2
    = \| \mathbf{A} \tilde{\boldsymbol{\phi}} - \mathbf{A} \boldsymbol{\phi} + \mathbf{A} \boldsymbol{\phi} - \mathbf{b} \|_2^2    
    = \| \mathbf{A} \tilde{\boldsymbol{\phi}} - \mathbf{A} \boldsymbol{\phi} \|_2^2 + \| \mathbf{A} \boldsymbol{\phi} - \mathbf{b} \|_2^2
\end{align}
where the second equality follows because $\mathbf{A} \boldsymbol{\phi} - \mathbf{b}$ is orthogonal to any vector in the span of $\mathbf{A}$.    
So to prove the theorem, it suffices to show that
\begin{align}
\| \mathbf{A \tilde{\boldsymbol{\phi}} - A \boldsymbol{\phi}} \|_2^2
\leq  \epsilon
\| \mathbf{A \boldsymbol{\phi} - b} \|_2^2.
\end{align}

Let $\mathbf{U} \in \mathbb{R}^{\rho \times n}$ be an orthonormal matrix that spans the columns of $\mathbf{A}$.
There is some $\mathbf{y}$ such that $\mathbf{Uy} = \mathbf{A} \boldsymbol{\phi}$ and some $\tilde{\mathbf{y}}$ such that $\mathbf{U}\tilde{\mathbf{y}} = \mathbf{A} \tilde{\boldsymbol{\phi}}$.
Observe that $\| \mathbf{A \tilde{\boldsymbol{\phi}} - A \boldsymbol{\phi}} \|_2 = \| \mathbf{U}\tilde{\mathbf{y}} - \mathbf{Uy} \|_2 = \| \tilde{\mathbf{y}} - \mathbf{y} \|_2$ where the last equality follows because $\mathbf{U}^\top \mathbf{U} = \mathbf{I}$.

By the reverse triangle inequality and the submultiplicavity of the spectral norm, we have
\begin{align}
\| \tilde{\mathbf{y}} - \mathbf{y} \|_2
&\leq \| \mathbf{U^\top S^\top S U} (\tilde{\mathbf{y}} - \mathbf{y}) \|_2
+ \| \mathbf{U^\top S^\top S U} (\tilde{\mathbf{y}} - \mathbf{y})  - (\tilde{\mathbf{y}} - \mathbf{y}) \|_2
\\&\leq \| \mathbf{U^\top S^\top S U} (\tilde{\mathbf{y}} - \mathbf{y}) \|_2
+ \| \mathbf{U^\top S^\top S U} - \mathbf{I} \|_2 \| \tilde{\mathbf{y}} - \mathbf{y} \|_2.
\end{align}
Because $\mathbf{U}$ has the same leverage scores as $\mathbf{A}$ and the number of rows sampled in $\mathbf{S}$ is within a constant factor of $m$, we can apply Lemma~\ref{lemma:spectral}: With $m = O(n \log \frac{n}{\delta})$, we have $\| \mathbf{U^\top S^\top S U} - \mathbf{I} \|_2 \leq \frac12$ with probability $1-\delta/2$.
So, with probability $1-\delta/2$,
\begin{align}\label{eq:diff_UTSTSU}
\| \tilde{\mathbf{y}} - \mathbf{y} \|_2
\leq 
2 \| \mathbf{U^\top S^\top S U} (\tilde{\mathbf{y}} - \mathbf{y}) \|_2.
\end{align}
Then
\begin{align}
    \| \mathbf{U^\top S^\top S U} (\tilde{\mathbf{y}} - \mathbf{y}) \|_2
    &= \left\| \mathbf{U^\top S^\top} \left( \mathbf{S U} \tilde{\mathbf{y}} - \mathbf{Sb} + \mathbf{Sb} - \mathbf{S U} \mathbf{y} \right) \right\|_2
    \\&= \left\| \mathbf{U^\top S^\top S} \left(\mathbf{U}\mathbf{y} - \mathbf{b} \right)\right\|_2   
\end{align}
where the second equality follows because $\mathbf{SU} \tilde{\mathbf{y}} -\mathbf{S b}$ is orthogonal to any vector in the span of $\mathbf{SU}$.
By similar reasoning, notice that $\mathbf{U^\top} (\mathbf{U y - b}) = \mathbf{0}$.
Then, as long as $m = O(\frac{n}{\delta \epsilon})$, we have
\begin{align}
    \left\| \mathbf{U^\top S^\top S} \left(\mathbf{U}\mathbf{y} - \mathbf{b} \right)\right\|_2 
    \leq \frac{\sqrt{\epsilon}}{2 \sqrt{n}} \| \mathbf{U} \|_F \| \mathbf{Uy - b} \|_2 \label{eq:frobenius_bound}
\end{align}
with probability $1-\delta/2$ by Lemma~\ref{lemma:frobenius}.
Since $\mathbf{U}$ has orthonormal columns, $\|\mathbf{U}\|_F^2 \leq n$.
Then, combining inequalities yields
\begin{align}
    \| \mathbf{A \tilde{\boldsymbol{\phi}} - A \boldsymbol{\phi}} \|_2^2 
    = \| \tilde{\mathbf{y}} - \mathbf{y} \|_2^2
    \leq 2 \| \mathbf{U^\top S^\top S U} (\tilde{\mathbf{y}} - \mathbf{y}) \|_2
    \leq \epsilon \| \mathbf{Uy - b} \|_2 = \epsilon \| \mathbf{A} \boldsymbol{\phi} - \mathbf{b} \|_2^2
\end{align}
with probability $1-\delta$.
\end{proof}

\clearpage
\section{Extension to Probabilistic Values}\label{app:prob_vals}

In this section, we consider extending Kernel Banzhaf to probabilistic values.
We will first establish how the optimal solution connects to the probabilistic values.
Next, we will prove guarantees on the approximate solution selected via leverage score sampling.

In the remainder of this section, 
let $\mathbf{A}\in \mathbb{R}^{2^n \times n}$ such that
\begin{align}
    [\mathbf{A}]_{S,i} = \begin{cases}
        \frac{p_{|S|-1}}{a_n} & \textrm{if } i \in S \\
        -\frac{p_{|S|}}{a_n} & \textrm{if } i \notin S \\
    \end{cases}
\end{align}
where $i \in [n]$ and we use $S \subseteq[n]$ as an index for ease of notation.
Let $\mathbf{b} \in \mathbb{R}^{2^n}$ be the vector representation of the set function $v$ i.e., $b_S = v(S)$.

\subsection{Equivalence to Probabilistic Values}

\unconstrainedreg*

\begin{proof}[Proof of Lemma \ref{lemma:unconstrained_regression}]
Our proof strategy is to explicitly compute the optimal solution
\begin{align}
\mathbf{x}^* = (\mathbf{A}^\top \mathbf{A})^{-1}\mathbf{A}^\top \mathbf{b}.
\end{align}
We begin with $\mathbf{A}^\top \mathbf{A}$.
In the computation, we will repeatedly use the symmetry of $\mathbf{A}^\top \mathbf{A}$; namely, all the diagonal entries are the same and all the non-diagonal entries are the same.
The diagonal entries for $i \in [n]$ are given by
\begin{align}\label{eq:on_diag}
    a_n^2 [\mathbf{A}^\top \mathbf{A}]_{i,i}
    &= \sum_{S \subseteq [n]: i \in S} p_{|S|-1}^2 + 
    \sum_{S \subseteq [n]: i \notin S} p_{|S|}^2
    = \sum_{\ell=1}^{n} \binom{n-1}{\ell-1} p_{\ell-1}^2+
    \sum_{\ell=0}^{n-1} \binom{n-1}{\ell} p_{\ell}^2
    \\&= 2 \sum_{\ell=0}^{n-1} \binom{n-1}{\ell} p_{\ell}^2
\end{align}
where the last equality follows by a change of variable.
Similarly, the off-diagonal entries for $i \neq j$ are given by 
\begin{align}\label{eq:off_diag}
    a_n^2 [\mathbf{A}^\top \mathbf{A}]_{i,j}
    &= \sum_{S \subseteq [n]: i,j \in S} p_{|S|-1}^2
    - 2\sum_{S \subseteq [n]: i \in S,j \notin S} p_{|S|-1}p_{|S|}
    + \sum_{S \subseteq [n]: i,j \notin S} p_{|S|}^2 
    \\ &= \sum_{\ell=2}^{n} \binom{n-2}{\ell-2} p_{\ell-1}^2
    - 2 \sum_{\ell=1}^{n-1} \binom{n-2}{\ell-1} p_{\ell-1}p_{\ell}
    + \sum_{\ell=0}^{n-2} \binom{n-2}{\ell} p_{\ell}^2
    \\ &= \sum_{\ell=1}^{n-1} \binom{n-2}{\ell-1} \left(p_\ell - p_{\ell-1}\right)^2
    = a_n b_n.
\end{align}
By the definition of $a_n$, we have
\begin{align}
    a_n &= 2 \sum_{\ell=0}^{n-1} \binom{n-1}{\ell} p_{\ell}^2
    - \sum_{\ell=1}^{n-1} \binom{n-2}{\ell-1} \left(p_\ell - p_{\ell-1}\right)^2 
    = a_n^2 [\mathbf{A}^\top \mathbf{A}]_{i,i}
    - a_n^2 [\mathbf{A}^\top \mathbf{A}]_{i,j}.
\end{align}
It follows that $[\mathbf{A}^\top \mathbf{A}]_{i,i}- [\mathbf{A}^\top \mathbf{A}]_{i,j}=\frac1{a_n}.$
Then
\begin{align}\label{eq:unconstrained_ata_equals}
    \mathbf{A}^\top \mathbf{A}
    &= ([\mathbf{A}^\top \mathbf{A}]_{i,i} - [\mathbf{A}^\top \mathbf{A}]_{i,j}) \mathbf{I} + [\mathbf{A}^\top \mathbf{A}]_{i,j} \mathbf{1} \mathbf{1}^\top
    = \frac1{a_n}(\mathbf{I} +b_n\mathbf{1} \mathbf{1}^\top).
\end{align}
By the Sherman-Morrison formula,
$(\mathbf{I}+b_n \mathbf{1} \mathbf{1}^\top)^{-1}
= \mathbf{I} - \frac{b_n}{1+n b_n} \mathbf{1}\mathbf{1}^\top$.
Then
\begin{align}\label{eq:new_ata_inv}
    (\mathbf{A}^\top \mathbf{A})^{-1}
    =a_n \left(\mathbf{I} - \frac{b_n}{1+n b_n} \mathbf{1}\mathbf{1}^\top\right).   
\end{align}

Next, we compute $\mathbf{A}^\top \mathbf{v}$.
The $i$th entry of this $n \times 1$ vector is given by
\begin{align}
    a_n[\mathbf{A}^\top \mathbf{v}]_i
    &= \sum_{S \subseteq [n]: i \in S} p_{|S|-1} v(S)
    -  \sum_{S \subseteq [n]: i \notin S} p_{|S|} v(S) 
    = \sum_{S \subseteq [n] \setminus \{i\}}
    p_{|S|} (v(S \cup \{i\}) - v(S)) = \phi_i.
\end{align}
Therefore $\mathbf{A}^\top \mathbf{v}=\frac1{a_n} \boldsymbol{\phi}$.
Finally,
\begin{align}
    \mathbf{x}^*
    &= (\mathbf{A}^\top \mathbf{A})^{-1}\mathbf{A}^\top \mathbf{b}
    = a_n \left(\mathbf{I} - \frac{b_n}{1+n b_n} \mathbf{1}\mathbf{1}^\top \right)
    \frac1{a_n} \boldsymbol{\phi}
    \nonumber
\end{align}
and, taking the inverse,
\begin{align}
    \phi&= \left(\mathbf{I} - \frac{b_n}{1+n b_n} \mathbf{1 1^\top}\right)^{-1} \mathbf{x}^*
    = (\mathbf{I} + b_n \mathbf{1 1^\top}) \mathbf{x}^{*}.
\end{align}
\end{proof}

\subsection{Approximation Guarantees}

\probguarantee*

We prove Theorem \ref{thm:prob_guarantee} by applying Theorem \ref{thm:full_leverage_score} to the generalized regression problem.
First, we compute the leverage scores of the generalized regression problem.

\begin{lemma}\label{lemma:unconstrained_leverage}
Let $S \subseteq [n]$ be an index.
The corresponding leverage score of the generalized regression problem is
\begin{align}
    \ell_S &= 
    \frac1{a_n} \biggl[
    |S| p_{|S|-1}^2 + (n-|S|) p_{|S|}^2 - \frac{b_n}{1+n b_n} \left(|S|p_{|S|-1} - (n-|S|) p_{|S|})^2\right) \biggr].
\end{align}
\end{lemma}

\begin{proof}[Proof of Lemma \ref{lemma:unconstrained_leverage}]
Using the definition of leverage scores and Equation \ref{eq:new_ata_inv}, we have
\begin{align}
    \ell_S
    &= [\mathbf{A}]_S^\top (\mathbf{A}^\top \mathbf{A})^{-1} [\mathbf{A}]_S
    = a_n [\mathbf{A}]_S^\top \left(\mathbf{I} - \frac{b_n}{1+n b_n} \mathbf{1}\mathbf{1}^\top\right) [\mathbf{A}]_S
    = a_n \left(\| [\mathbf{A}]_S \|_2^2 - \frac{b_n}{1+n b_n} \| [\mathbf{A}]_S \|_1^2 \right).
\end{align}
By the definition of $\mathbf{A}$,
the $\ell_2$-norm squared is given by
\begin{align}
    \| [\mathbf{A}]_S \|_2^2
    &= \frac1{a_n^2} \left( \sum_{i \in S} p_{|S|-1}^2 + \sum_{i \notin S} p_{|S|}^2 \right)  
    = \frac1{a_n^2} \left( |S| p_{|S|-1}^2 + (n-|S|) p_{|S|}^2 \right)
\end{align}
while the $\ell_1$-norm squared is given by
\begin{align}
    \| [\mathbf{A}]_S \|_1^2
    &= \frac1{a_n^2} \left( \sum_{i \in S} p_{|S|-1} - \sum_{i \notin S} p_{|S|} \right)^2
    = \frac1{a_n^2} \left( |S| p_{|S|-1} - (n-|S|) p_{|S|} \right)^2.
\end{align}
The lemma statement follows.
\end{proof}

Next, we prove Theorem \ref{thm:prob_guarantee}.

\begin{proof}[Proof of Theorem \ref{thm:prob_guarantee}]
Let $\gamma = \|\mathbf{Ax}^* - \mathbf{b}\|_2^2/\| \mathbf{Ax}^*\|_2^2$ and
suppose there are $m = O(n \log (n/\delta) + n \log(1/\delta) / \epsilon)$ samples.
By the standard leverage score sampling guarantee in Theorem \ref{thm:full_leverage_score} (see e.g., the first few lines of the proof of Corollary \ref{coro:l2norm} for the step-by-step process), we have
\begin{align}
    \| \mathbf{A} \tilde{\mathbf{x}} - \mathbf{A}\mathbf{x}^* \|_2^2 \leq \epsilon \| \mathbf{A}\mathbf{x}^* - \mathbf{b} \|_2^2 = \epsilon \gamma \| \mathbf{A}\mathbf{x}^* \|_2^2
\end{align}
with probability $1-\delta$.
For the data matrix $\mathbf{A}$ in the unconstrained regression problem, we have
\begin{align}
    \| \mathbf{Ax} \|_2^2 = \mathbf{x^\top A^\top A x} = \frac1{a_n} \mathbf{x}^\top (\mathbf{I} + b_n \mathbf{11}^\top) \mathbf{x} = \frac1{a_n} \left[\| \mathbf{x}\|_2^2 + b_n (\mathbf{1^\top x})^2\right]
\end{align}
where the second equality follows by Equation \ref{eq:unconstrained_ata_equals}.

Then the leverage score sampling guarantee implies
\begin{align}
    \frac1{a_n} \left[\| \mathbf{x}^* - \tilde{\mathbf{x}}\|_2^2 + b_n (\mathbf{1}^\top (\mathbf{x}^* - \tilde{\mathbf{x}}))^2\right] \leq \epsilon \gamma
    \frac1{a_n} \left[\| \mathbf{x}^* \|_2^2 + b_n (\mathbf{1}^\top \mathbf{x}^*)^2\right]
\end{align}
so
\begin{align}
    \| \mathbf{x}^* - \tilde{\mathbf{x}}\|_2^2 
    \leq \epsilon \gamma \left(
    \| \mathbf{x}^* \|_2^2 + b_n (\mathbf{1}^\top \mathbf{x}^*)^2 \right)
    \leq \epsilon \gamma (1 + n b_n) \| \mathbf{x}^*\|_2^2
\end{align}
where the last inequality follows because $(\mathbf{1}^\top \mathbf{x}^*)^2 \leq \| \mathbf{1} \|_2^2\| \mathbf{x}^* \|_2^2 = n \| \mathbf{x}^* \|_2^2$ by Cauchy-Schwarz.

Since $b_n$ is non-negative, observe that $\|\mathbf{I} + b_n \mathbf{11}^\top \|_2 = 1+n b_n $ and $\|(\mathbf{I} + b_n \mathbf{11}^\top)^{-1} \|_2 = 1$.
It follows that
\begin{align}
  \| \mathbf{x}^*\|_2^2 = \|(\mathbf{I} + b_n \mathbf{11}^\top)^{-1}(\mathbf{I} + b_n \mathbf{11}^\top)  \mathbf{x}^*\|_2^2  
  = \|(\mathbf{I} + b_n \mathbf{11}^\top)^{-1} \boldsymbol{\phi}\|_2^2
    \leq  \|(\mathbf{I} + b_n \mathbf{11}^\top)^{-1}\|_2^2 \cdot \| \boldsymbol{\phi} \|_2^2 = \| \boldsymbol{\phi} \|_2^2.
\end{align}
Then
\begin{align}
\| \boldsymbol{\phi} - \tilde{\boldsymbol{\phi}} \|_2^2
&= \| (\mathbf{I} + b_n \mathbf{11}^\top) ( \mathbf{x}^* - \tilde{\mathbf{x}}) \|_2^2
\\&\leq \| \mathbf{I} + b_n \mathbf{11}^\top \|_2^2 \cdot \|  \mathbf{x}^* - \tilde{\mathbf{x}} \|_2^2
\\&\leq \epsilon \gamma(1+n b_n)^3  \| \mathbf{x}^*\|_2^2
\\&\leq \epsilon \gamma(1+n b_n)^3  \| \boldsymbol{\phi}\|_2^2
\end{align}
\end{proof}

\clearpage
\section{Datasets and Models}
\label{appendix:dataset}

\subsection*{Datasets}
The Diabetes dataset \citep{bache2013uci}, sourced from the National Institute of Diabetes and Digestive and Kidney Diseases, comprises 8 features. Its primary objective is to predict, based on diagnostic measurements, whether a patient has diabetes, thus it's categorized as a classification task. 
The Census Income dataset \citep{bache2013uci, covert2020improving}, also known as the Adult dataset, involves predicting whether an individual's income exceeds \$50K/yr based on census data, using 14 features. 
The Portuguese Bank Marketing dataset \citep{moro2014data} is another classification task with 16 features aimed at predicting client subscription to a term deposit. 
The German Credit dataset \citep{bache2013uci}, known as Statlog, involves classifying individuals as having good or bad credit risks based on 20 attributes. 
The NHANES dataset, with 79 features derived from the National Health and Nutrition Examination Survey (NHANES) I Epidemiologic Followup Study, models the risk of death over a 20-year follow-up period, as discussed in \citep{lundberg2020local, karczmarz2022improved}. 
For the Breast Cancer (BRCA) subtype classification dataset, 100 out of 17,814 genes were selected to minimize overfitting in a relatively small dataset of 510 patients, following guidelines from \citep{covert2020improving}. 
The Communities and Crime Unnormalized dataset \citep{bache2013uci} aims to predict the total number of violent crimes per 100,000 population, comprising a predictive regression task with 101 features. The Tezpur University Android Malware Dataset (TUANDROMD) \citep{bache2013uci} includes 241 attributes, with the primary classification target distinguishing between malware and goodware.

These datasets vary in size and column types and are predominantly utilized in previous studies for semi-value-based model explanation \citep{lundberg2017unified,covert2020improving, lundberg2020local, karczmarz2022improved}. We primarily focus on tabular datasets because they are more thoroughly studied in this field and allow for easier acquisition of ground truth, especially in large datasets, using tree-based algorithms. Additionally, tabular datasets are prevalent in scenarios involving smaller datasets with fewer features.

\subsection*{Models}
For the experiments involving tree set functions, we trained an XGBoost regressor model \citep{chen2016xgboost} with 100 trees and a maximum depth of 4. For the non-tree model experiments, we utilized a two-layer neural network equipped with a dropout layer with a rate of 0.5 to mitigate overfitting. This network was trained using a batch size of 32 and a learning rate of 0.0001, across 100 epochs. We chose this relatively simple model architecture because our primary focus is on explaining model behavior rather than maximizing its predictive accuracy.

There are generally two approaches to handling removed features in feature perturbation for general set functions, as discussed in \citet{chen2020true} and \citet{kumar2020problems}. Given an explicand $x$ and a subset of features $S$,
define $\mathbf{x}^S$ as the observation where $\mathbf{x}^S_i = \mathbf{x}_i$ if feature $i \in S$ and, otherwise, $\mathbf{x}^S_i$ is sampled from one of two distributions.
The first method involves sampling from the conditional distribution of the removed features. This approach, while precise, is computationally expensive. Alternatively, the marginal distribution can be used where the observed features $\mathbf{x}^S_i$ for $i \in S$ are ignored. Due to its lower computational complexity, we adopt the latter approach.

%
For \emph{tree-based models}, as we need to use the tree-based algorithm for calculating the ground truth Banzhaf values to measure errors, we utilize a method aligned with Algorithm 1 from \citet{karczmarz2022improved}, which computes predictions using partial features. Specifically, during tree traversal, if feature $i \in S$, we proceed according to the threshold to select the child node; if $i \notin S$, we traverse both children and compute a weighted average of the predictions, effectively nullifying the influence of features not in $S$ without any feature value replacement. 
We use XGBoost \citep{chen2016xgboost} for our tree-based models, which makes it possible to apply the tree-based algorithm of \citep{karczmarz2022improved} to compute exact Banzhaf values for $n \geq 30$. 

For \emph{neural network models}, instead of using fixed baseline values for the removed features, we compute the average of the model’s predictions using replacement values randomly sampled from 50 baseline points, different from the explicand. For each explicand $\mathbf{x}$, the non-selected features in $\mathbf{x}^S$ are replaced by values from baseline points, and the average of $M(\mathbf{x}^S)$ is taken to estimate the impact of marginalizing out the non-selected features. To calculate ground truth Banzhaf values, we evaluate all $2^n$ subsets of features in this way.


\clearpage
\section{Neural Network Models}
\label{app:neural_network}

In this section, we evaluate Banzhaf estimators for explaining neural network models, where the output of the set function is the neural network's raw prediction. Calculating ground truth Banzhaf values for datasets with more than 50 features presents challenges, primarily because the tree-based algorithm for Banzhaf value calculation, as proposed in \citet{karczmarz2022improved}, is limited to tree models. Consequently, our experiments are confined to four smaller datasets. As illustrated in Figure~\ref{fig:l2-by-sample-size-nn}, Kernel Banzhaf estimators, both with and without paired sampling, consistently outperform Monte Carlo (MC) and Maximum Sample Reuse (MSR) estimators in non-tree set functions. These experiments further underscore the generalizability of our algorithm across different model types.

\begin{figure}[h]
    \centering
    \includegraphics[width=\linewidth]{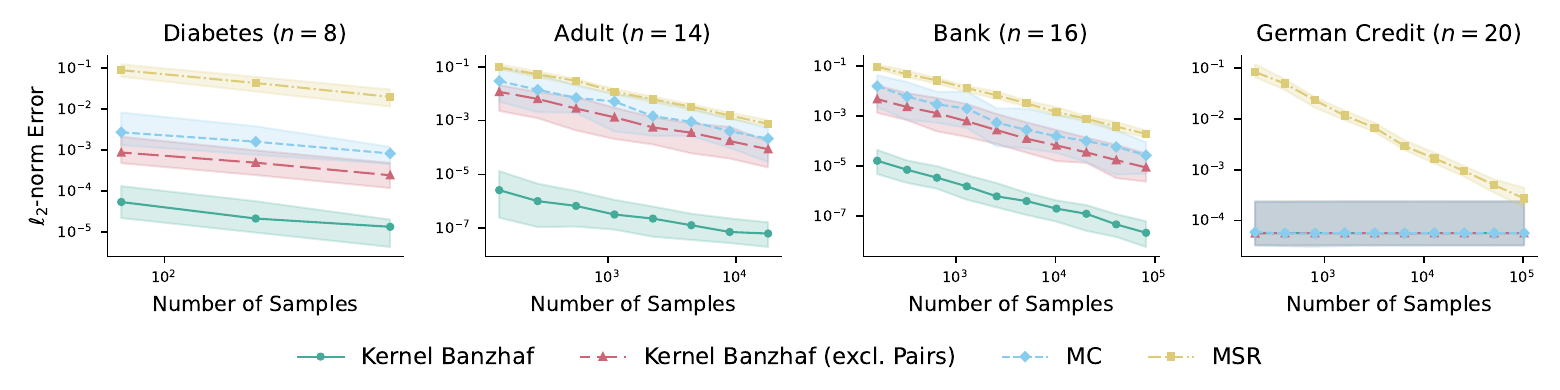}
    \caption{This figure compares the relative squared $\ell_2$-norm errors of Kernel Banzhaf (including an ablated version without paired sampling), MC, and MSR across increasing sample sizes on four small datasets in explaining neural network models. The results highlight the robust performance and generalizability of Kernel Banzhaf across various model types.}
    \label{fig:l2-by-sample-size-nn}
\end{figure}

\clearpage
\section{Alternative Error Metric: Regression Objective Error}\label{app:objective_error}

To further demonstrate the superior performance of Kernel Banzhaf, we evaluated the Banzhaf estimators against exact Banzhaf values using relative objective error in Figure \ref{fig:objective}. Building on the framework where Banzhaf values solve a linear regression problem (as established in Theorem~\ref{thm:equivalence}), we aim to minimize the objective function $\| \mathbf{A x-b} \|_2$. Here, $\phivec$ represents the vector of exact Banzhaf values for each feature, and $\approxphivec$ represents the estimations from different estimators. We evaluate the differences between $\| \mathbf{A \phivec-b} \|_2$ and $\| \mathbf{A \approxphivec-b} \|_2$, because although the optimal linear regression solution corresponds to the Banzhaf values, the optimal objective error $\| \mathbf{A \phivec-b} \|_2$ is not necessarily zero. Thus, we assess the relative objective error. The comparative analysis, shown in Figure~\ref{fig:objective}, reveals that the plots of this error with increasing sample size and noise level are analogous to those using relative squared $\ell_2$-norm error, where Kernel Banzhaf consistently surpasses all other estimators in both efficacy and robustness. These experiments underscore the effectiveness of our Kernel Banzhaf algorithm in accurately estimating Banzhaf values and solving the corresponding linear regression challenge.

\begin{figure}[h]
    \centering
    \includegraphics[width=\linewidth]{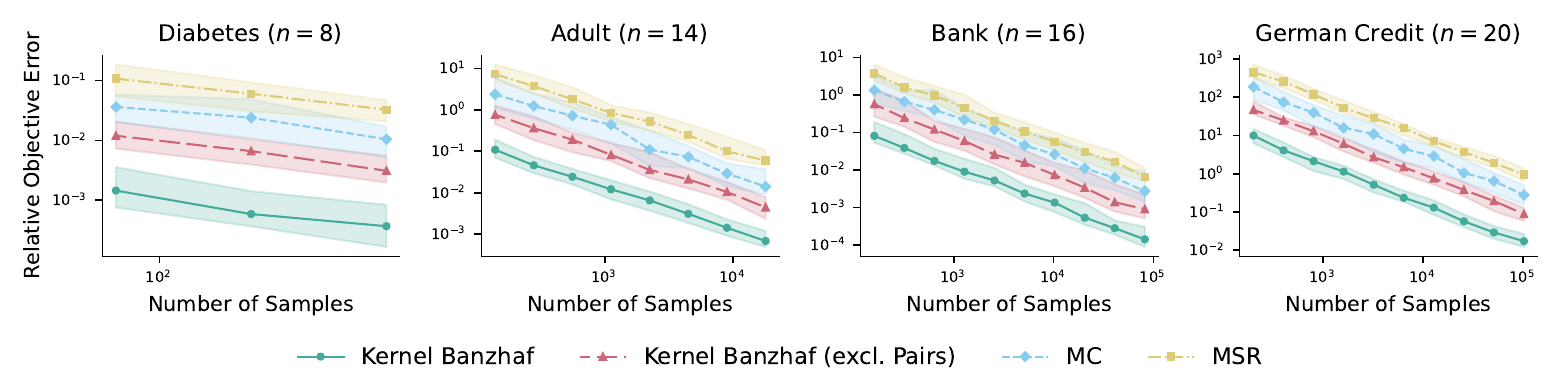}
    \includegraphics[width=\linewidth]{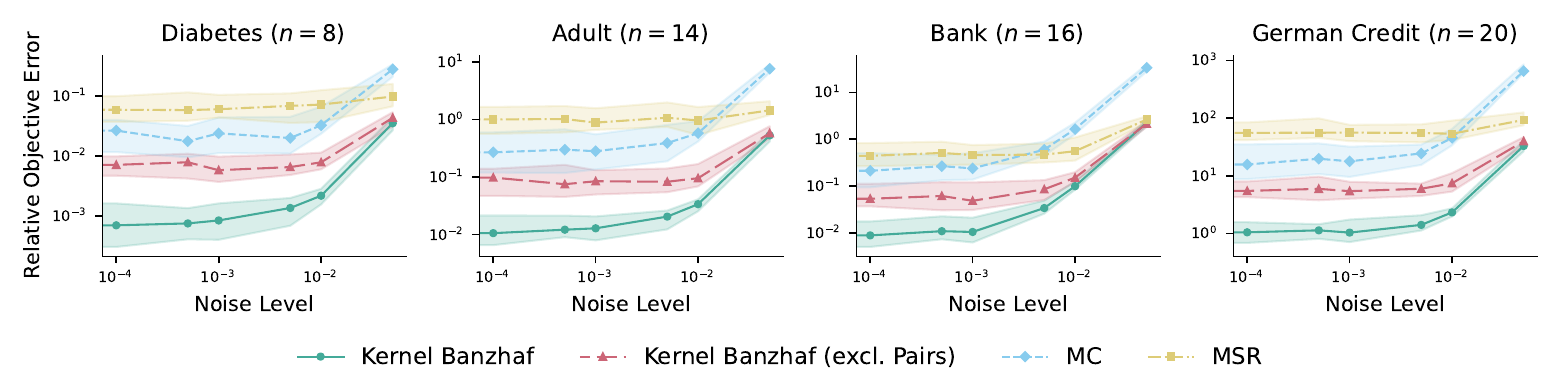}
    \caption{Relative objective error $\frac{\| \mathbf{A \approxphivec-b} \|_2 - \| \mathbf{A \phivec-b} \|_2}{\| \mathbf{A \phivec-b}\|_2}$ plotted against sample size (top row) and noise level (bottom row) for comparing Banzhaf estimators across four small datasets. Kernel Banzhaf consistently outperforms other estimators, demonstrating enhanced accuracy and robustness as sample sizes and noise levels vary. This performance highlights its efficacy in solving the linear regression problem where the solutions represent Banzhaf values.}
    \label{fig:objective}
\end{figure}

\clearpage

\section{Time Complexity}\label{app:time_complexity}

We further evaluated the computational efficiency of the MC, MSR, and Kernel Banzhaf estimators by measuring the exact time required to estimate Banzhaf values, as depicted in Figure~\ref{fig:time}. All experiments were conducted on a Lenovo SD650 with 128 GB of RAM, using only one thread for computation. Figure~\ref{fig:time} reveals that 
all four estimators exhibit comparable computational times across all datasets and sample sizes, since evaluating the set function $v$ dominates the estimator's time complexity.
\begin{figure}[h]
    \centering
    \includegraphics[width=\linewidth]{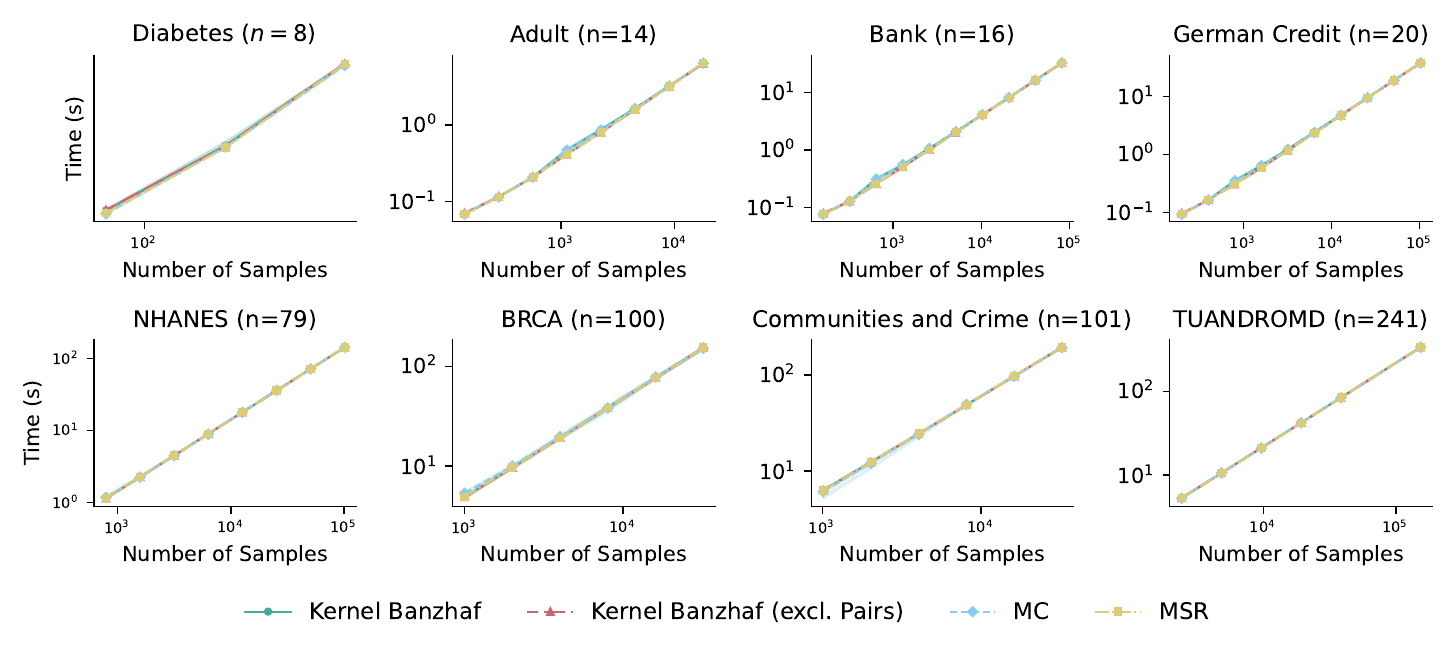}
    \caption{Computational time required for Banzhaf value estimation across varying sample sizes for eight datasets. 
    All estimators demonstrate comparable computational efficiencies primarily driven by the time required for model prediction.
    }
    \label{fig:time}
\end{figure}

\clearpage
\section{Ablation: Kernel Banzhaf With and Without Replacement}
\label{appendix:swor}

In order to ablate its sampling technique, we compare Kernel Banzhaf with and without replacement.
Our empirical results, as illustrated in Figure~\ref{fig:swor_vs_swr}, indicate that sampling without replacement does not yield improvements unless $m \geq 2^n$, in which case we can read the whole set function $v$ and exactly compute the Banzhaf values.
We suspect the reason that Kernel Banzhaf does not benefit from sampling without replacement is because it uniformly samples subsets to solve the regression problem.

\begin{figure}[h]
    \centering
    \includegraphics[width=\linewidth]{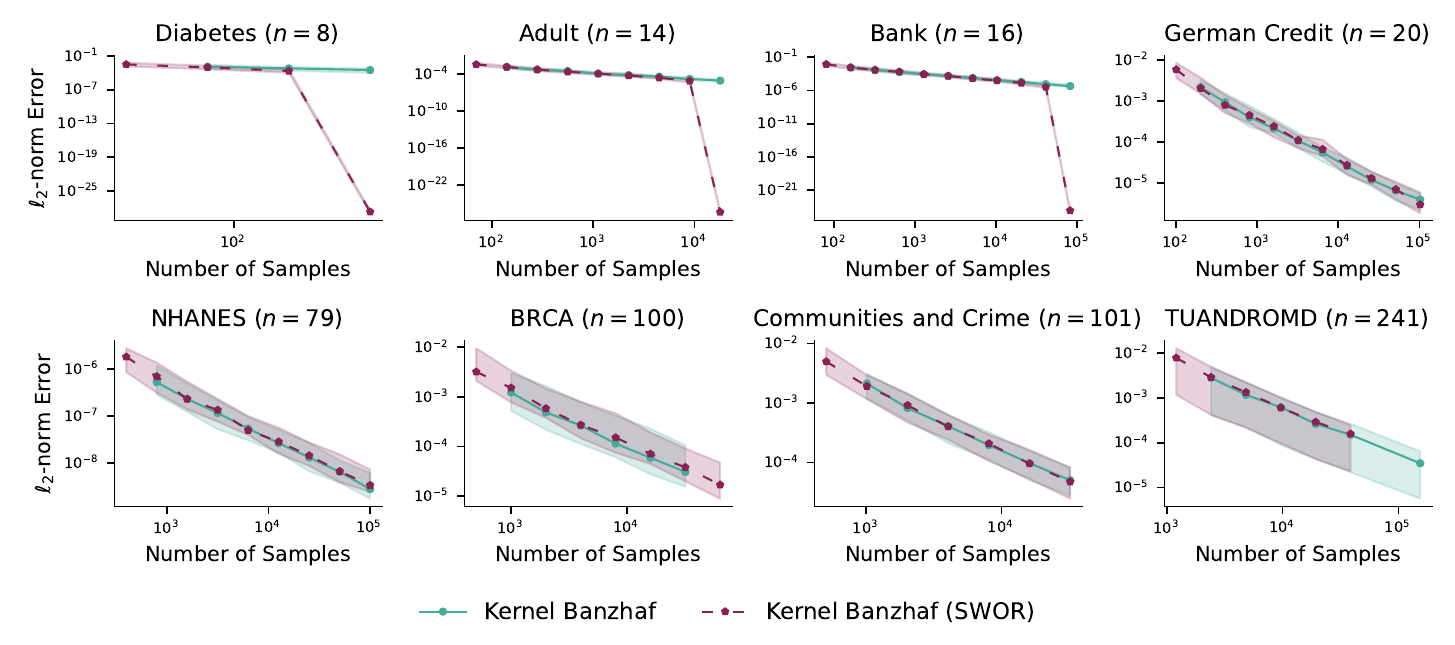}
    \caption{Comparison of relative squared $\ell_2$-norm error between Kernel Banzhaf and Kernel Banzhaf with sampling without replacement (SWOR) across eight datasets. When sample size $m < 2^n$, the performance of the two algorithms is quite similar, as illustrated by the overlapping lines. Note that the sharp error reduction in the first three small datasets occurs when sample size $m$ exceeds $2^n$, due to the SWOR algorithm exhaustively sampling all possible subsets at $m > 2^n$.}
    \label{fig:swor_vs_swr}
\end{figure}

\clearpage
\section{Feature Ranking Recovery}\label{app:feature_ranking}


{Aside from evaluating the quantitative errors between estimated Banzhaf values and exact Banzhaf values, another critical metric that reveals the meaningfulness of the estimated results is how well the estimator recovers feature ranking. Feature ranking is important to feature comparison and selection, which are useful for enhancing the performance of machine learning models. Accurate feature ranking helps in identifying the most influential features, thereby facilitating more efficient and effective feature engineering and dimensionality reduction strategies.}

{In order to evaluate this property, we incorporate two well-known metrics: \textit{Cayley distance} and \textit{Spearman rank correlation}. The Cayley distance refers to the minimum number of transpositions required to transform one permutation into another. This metric provides a concrete measure of the difference between two rankings, capturing the minimal edit sequence needed, which is particularly useful in understanding the stability and reliability of feature ranking methods, and it's also adopted in \citet{karczmarz2022improved}.
Spearman's rank correlation, $\rho$, on the other hand, measures the strength and direction of association between two ranked variables. 
Formally, it is defined as the Pearson correlation coefficient between the rank values of the variables, mathematically expressed as:
\[
\rho = 1 - \frac{6 \sum d_i^2}{n(n^2 - 1)}
\]
where \(d_i\) represents the difference between the ranks of corresponding variables \(x_i\) and \(y_i\), and \(n\) is the number of observations.
This metric offers insights into how well the ranking produced by the estimator preserves the monotonic relationship compared to the exact ranking, providing a measure of ranking fidelity.}

{In scenarios with a large feature space, the most significant features often have a more pronounced impact on model predictions. In these cases, the overall ranking may be cluttered with a large number of features that show only minor differences in their Banzhaf values, making it difficult to distinguish among lower-ranked features effectively. 
Focusing on the top 20\% of features, therefore, targets those variables most likely to affect predictive accuracy and model stability, offering a more pragmatic evaluation of ranking recovery.}

{In our evaluation, we identify the top 20\% of features ranked by exact Banzhaf values. For any feature missing from the top 20\% as derived from the estimated Banzhaf values, we add it to the end of the permutation. This adjustment ensures that the evaluation penalizes discrepancies at the top of the distribution, which are most critical for decision-making and model interpretation.}


{Our findings for the top 20\% feature ranking recovery are presented in Figures~\ref{fig:topk_cayley} and \ref{fig:topk_spearman}. This task, focusing only on the ranks of the most significant features, is more practical. For datasets with feature size $\leq 20$, we take the top 7 features instead. Our Kernel Banzhaf algorithm consistently outperforms or matches the MC method across all datasets, demonstrating its effectiveness in identifying and ranking important features.}

We also explore feature ranking recovery on \textbf{all} the features in Figures \ref{fig:cayley} and \ref{fig:spearman}.
MC gives the best performance, followed by Kernel Banzhaf. Because of its sampling structure i.e., $v(S \cup \{i\})-v(S)$, MC is uniquely suited to recover Banzhaf values near 0. As shown in Figure \ref{fig:hist}, there are many Banzhaf values near 0 and hence MC gives the best performance when recovering these features.

\begin{figure}[h]
    \centering
    \includegraphics[width=\linewidth]{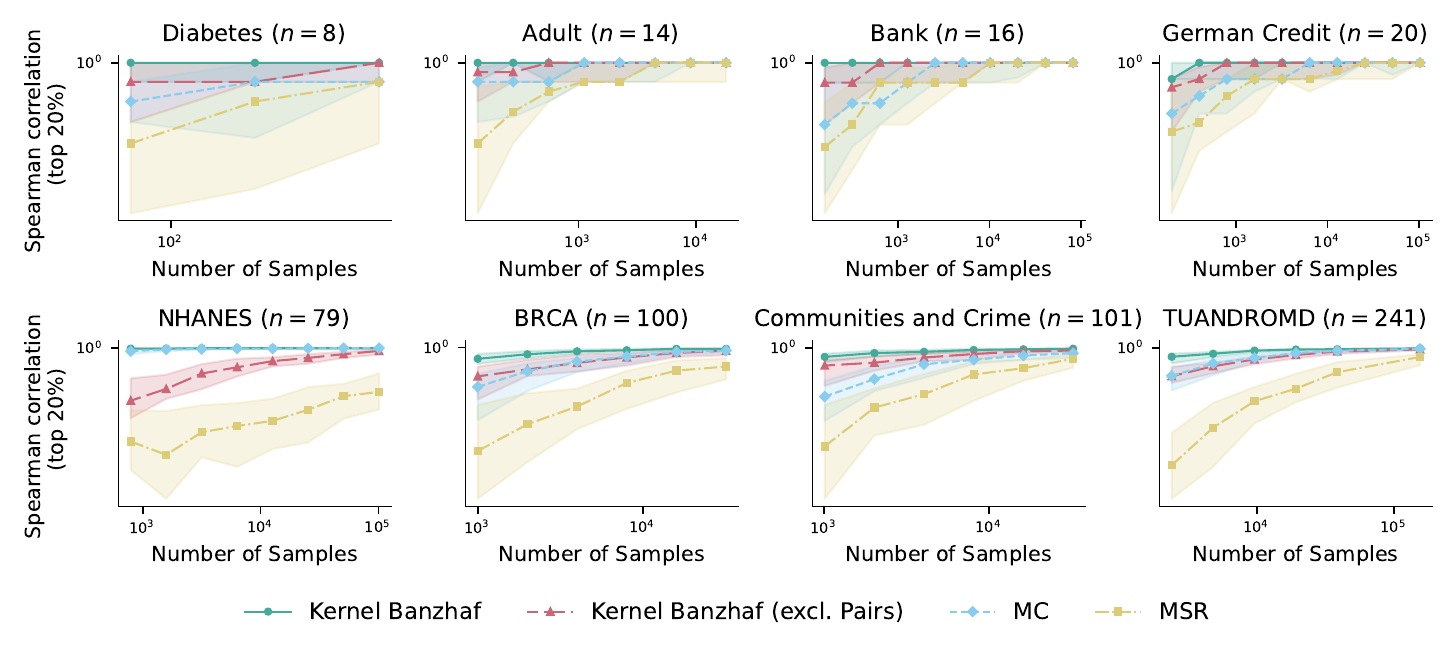}
    \caption{Comparison of top 20\% feature ranking recovery using Spearman correlation (higher is more accurate). Kernel Banzhaf consistently gives the best performance.}
    \label{fig:topk_spearman}
\end{figure}

\begin{figure}[h]
    \centering
    \includegraphics[width=\linewidth]{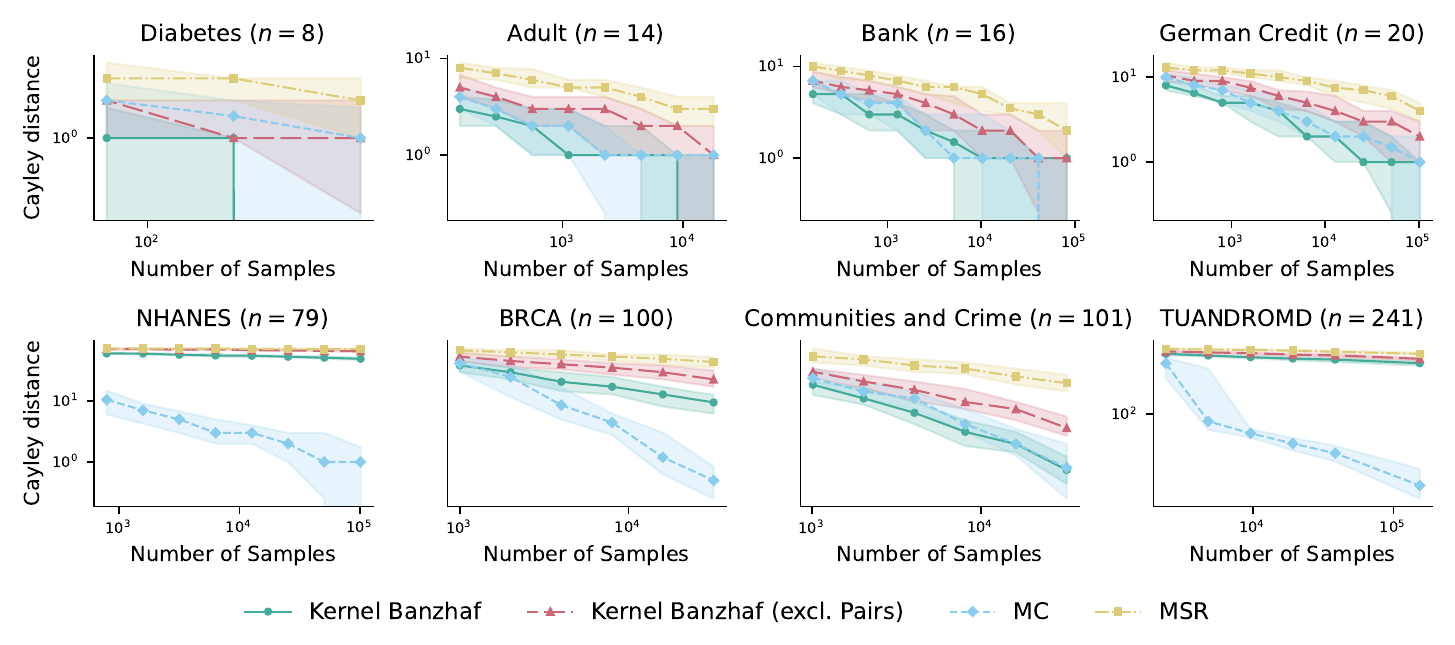}
    \caption{Comparison on \textbf{full} feature ranking recovery using Cayley distance (lower is more accurate). MC gives the best performance, followed by Kernel Banzhaf. Because of its sampling structure i.e., $v(S \cup \{i\})-v(S)$, MC is uniquely suited to recover Banzhaf values near 0. As shown in Figure \ref{fig:hist}, there are many Banzhaf values near 0.}
    \label{fig:cayley}
\end{figure}

\begin{figure}[h]
    \centering
    \includegraphics[width=\linewidth]{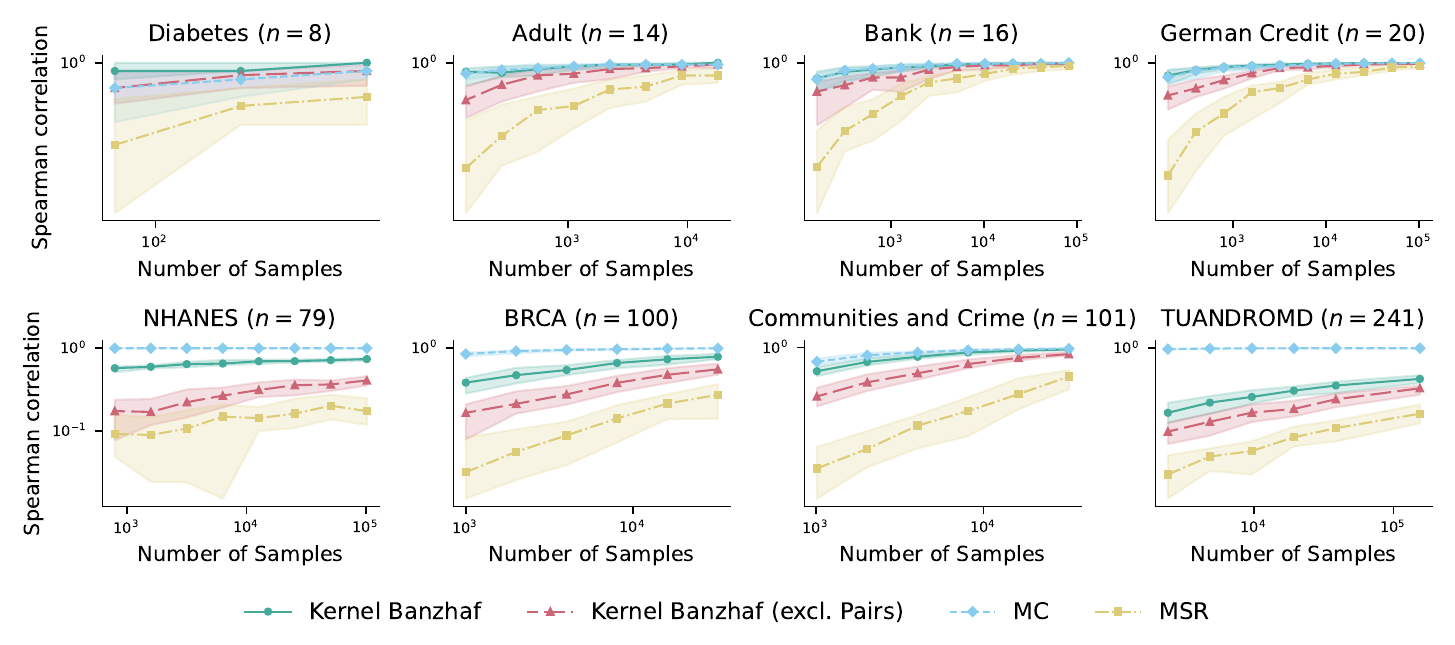}
    \caption{Comparison on \textbf{full} feature ranking recovery using Spearman correlation (higher is more accurate). MC gives the best performance, followed by Kernel Banzhaf. Because of its sampling structure i.e., $v(S \cup \{i\})-v(S)$, MC is uniquely suited to recover Banzhaf values near 0. As shown in Figure \ref{fig:hist}, there are many Banzhaf values near 0.}
    \label{fig:spearman}
\end{figure}

\begin{figure}
    \centering
    \includegraphics[width=\linewidth]{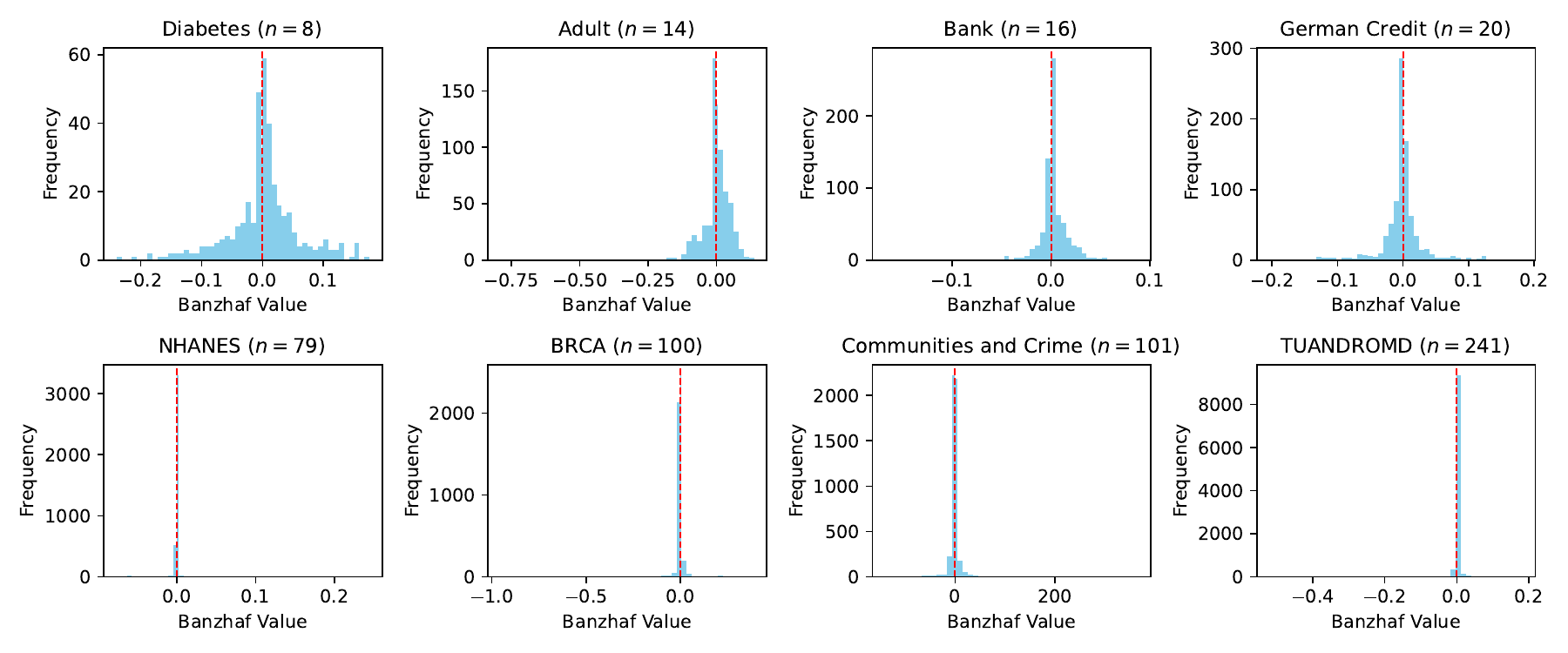}
    \caption{There are many Banzaf values near 0, giving MC an advantage if full feature ranking recovery.}
    \label{fig:hist}
\end{figure}

\clearpage
\section{Adversarial Perturbations}\label{app:adversarial_perturbations}

{To further demonstrate the robustness of our proposed estimator, instead of independently perturbing all queries to the set function, we only perturb sets $S$ that contain a chosen item $i$. In the first experiment setting, we select $i$ uniformly at random. Then, instead of observing $v(S)$ on the query to subset $S$, the algorithms observe $v(S) + x$ where $x$ 0 if $i \not \in S$ and, if $i \in S$, we have $x \sim \mathcal{N}(0, \sigma^2)$ as before. These results appear in Figure~\ref{fig:adver-noise}.}

\begin{figure}[h]
    \centering
    \includegraphics[width=\linewidth]{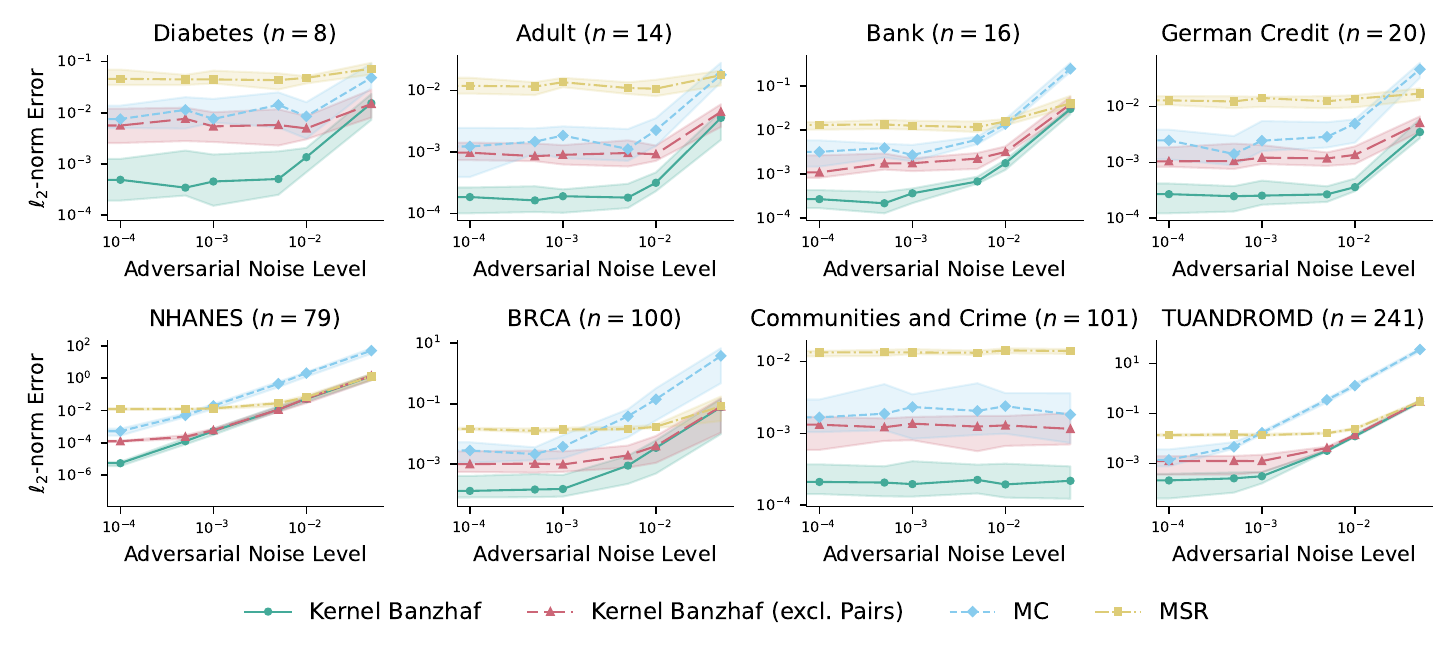}
    \caption{Plots of relative squared $\ell_2$-norm error by adversarial noise levels across Banzhaf estimators. For each noise level $\sigma$, the estimator observes $v(S) + x$ where $x \sim N(0,\sigma)$ when \textbf{set $S$ contains a randomly chosen feature $i$}.
    Similar to previous experiments, Kernel Banzhaf outperforms for lower noise levels, eventually matching its ablated version and MSR for larger noise.}
    \label{fig:adver-noise}
\end{figure}

{Beyond this, in a follow-up experiment, we introduce a more adversarial noise setting. Each algorithm is run once on set function $v$ (no perturbation in the query access). Then we compute the relative error of each estimated value $\tilde{\phi}_j$ relative to the baseline $\phi_j$. We select the item $i$ with the largest relative error. Then, we evaluate each algorithm as before, but now the queries are perturbed if the set $S$ contains the adversarially chosen $i$. These results appear in Figure~\ref{fig:adver-plus-noise}}.

\begin{figure}[h]
    \centering
    \includegraphics[width=\linewidth]{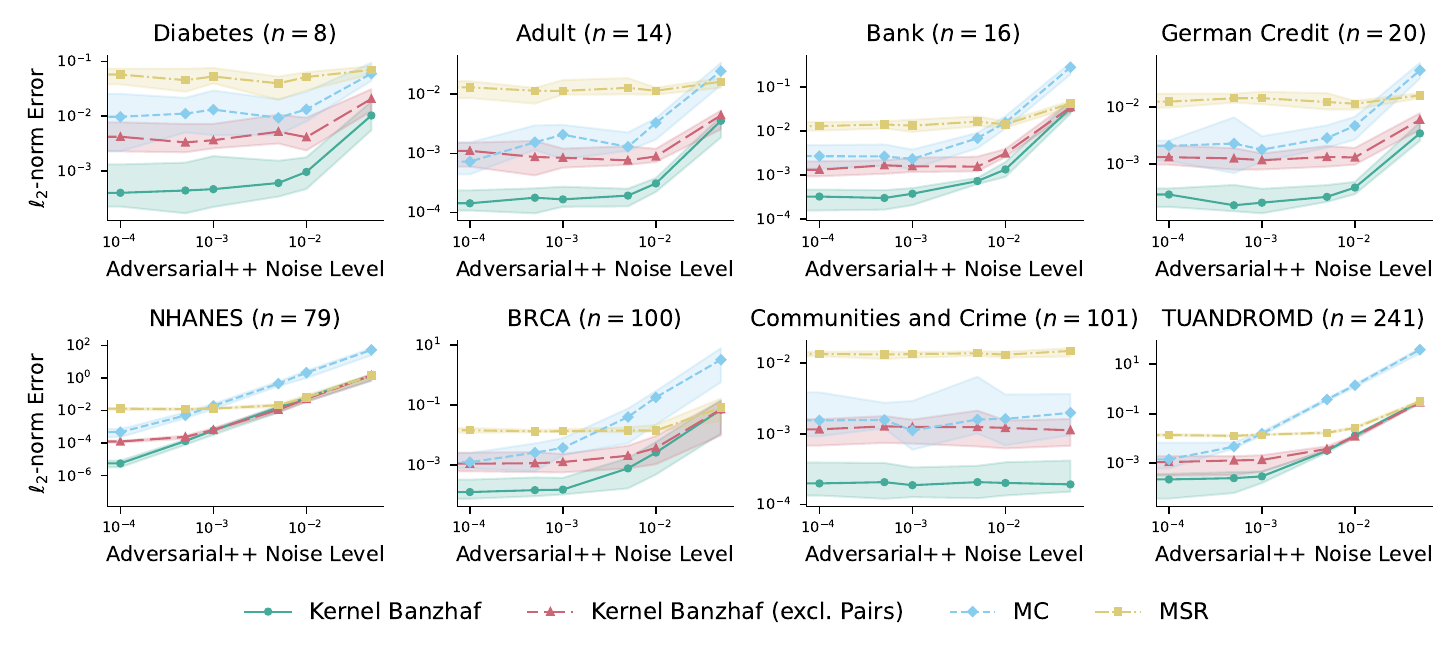}
    \caption{{Plots of relative squared $\ell_2$-norm error by adversarial noise levels across Banzhaf estimators. For each noise level $\sigma$, the estimator observes $v(S) + x$ where $x \sim N(0,\sigma)$ when \textbf{set $S$ contains feature $i$ with the largest relative error in the normal setting}.
    Similar to previous experiments, Kernel Banzhaf outperforms for lower noise levels, eventually matching its ablated version and MSR for larger noise. }}
    \label{fig:adver-plus-noise}
\end{figure}

\end{document}